\documentclass[
aps,%
12pt,%
final,%
notitlepage,%
oneside,%
onecolumn,%
nobibnotes,%
nofootinbib,%
superscriptaddress,%
noshowpacs,%
centertags]%
{revtex4}

\usepackage{url}
\usepackage[algoruled,linesnumbered]{algorithm2e}
\usepackage{subfigure}
\usepackage{amsthm}

\newtheorem{proposition}{Proposition}

\newcommand{\eqdef}{\buildrel \text{d{}ef}\over = }
\newcommand{\argmin}{\mathop{\mathrm{argmin}}}
\newcommand{\sgn}{\mathop{\mathrm{sgn}}}
\newcommand{\myldots}{\;\ldots\;}
\newcommand{\vbeta}{\boldsymbol \beta}
\newcommand{\va}{\mathbf a}

\newcommand{\vx}{\mathbf x}
\newcommand{\vy}{\mathbf y}
\newcommand{\hy}{\hat{y}}

\newcommand{\suchthat}{\;\ifnum\currentgrouptype=16 \middle\fi|\;}

\renewcommand{\tilde}{\widetilde}

\begin{document}
\selectlanguage{english}

\title{Distributed Coordinate Descent for Generalized Linear Models with Regularization}

\author{\firstname{Ilya}~\surname{Trofimov}}
\email{trofim@yandex-team.ru}
\affiliation{%
Yandex
}%
\author{\firstname{Alexander}~\surname{Genkin}}
\email{alexander.genkin@gmail.com}
\affiliation{%
NYU Langone Medical Center
}%


\begin{abstract}
Generalized linear model with $L_1$ and $L_2$ regularization is a widely used technique for solving classification, class probability estimation and regression problems.
With the numbers of both features and examples growing rapidly in the fields like text mining and clickstream data analysis parallelization and the use of cluster architectures becomes important.
We present a novel algorithm for fitting regularized generalized linear models in the distributed environment.
The algorithm splits data between nodes by features, uses coordinate descent on each node and line search to merge results globally.
Convergence proof is provided.
A modifications of the algorithm addresses slow node problem.
For an important particular case of logistic regression we empirically compare our program with several state-of-the art approaches that rely on different algorithmic and data spitting methods.
Experiments demonstrate that our approach is scalable and superior when training on large and sparse datasets.
\end{abstract}

\maketitle


\section{Introduction}
Generalized linear model (GLM) with regularization is the method of choice for solving classification, class probability estimation and regression
problems in text classification \citep{Genkin2007}, clickstream data analysis \cite{Mcmahan2013}, web data mining \citep{Yuan2010} and compressed sensing \cite{Bradley2011}.
Despite the fact that GLM can build only linear separating surfaces and regressions, with proper regularization it can achieve good testing accuracy for high dimensional input spaces.
For several problems the testing accuracy has shown to be close to that of nonlinear kernel methods \citep{Yuan2012Recent}.
At the same time training and testing of linear predictors is much faster.
It makes the GLM a good choice for large-scale problems.

Choosing the right regularizer is problem dependent. $L_2$-regularization is known to shrink coefficients towards zero leaving correlated ones in the model.
$L_1$-regularization leads to a sparse solution and typically selects only one coefficient from a group of correlated ones.
Elastic net regularizer is a linear combination of $L_1$ and $L_2$. It allows to select a trade-off between them.
Other regularizers are less often used: group lasso \cite{yuan2006model, Meier2008} which includes and excludes variables in groups, non-smooth bridge \cite{Fu1998} and non-convex SCAD \cite{Fan2001} regularizers.

Fitting a statistical model on a large dataset is time consuming and requires a careful choosing of an optimization algorithm.
Not all methods working on a small scale can be used for large scale problems. At present time algorithms dedicated for optimization on a single machine are well developed.

Fitting commonly used GLMs with $L_2$ regularization is equivalent to minimization of a smooth convex function.
On the large scale this problem is typically solved by the conjugate gradient method, ``Limited memory BFGS'' (L-BFGS) \citep{Sutton2002, nocedal1980updating},
TRON \cite{Zhuang2015}.
Coordinate descent algorithms works well in primal \citep{Genkin2007} and dual \citep{Yu2011}.
Also this problem can be effectively solved by various online learning algorithms \citep{Mcmahan2013, Karampatziakis2010}.

Using $L_1$ regularization is harder because it requires to optimize a convex but non-smooth function.
A broad survey \cite{Yuan2010} suggests that coordinate descent methods are the best choice for $L_1$-regularized logistic regression on the large scale.
Widely used algorithms that fall into this family are: BBR \cite{Genkin2007}, GLMNET \cite{Friedman2010}, newGLMNET \cite{Yuan2012a}. Coordinate descent methods also work well for large-scale high dimensional LASSO \cite{Fu1998}.
Software implementations of these methods start with loading the full training dataset into RAM, which limits the possibility to scale up.

Completely different approach is online learning \citep{Balakrishnan2007, Langford2009, Mcmahan2011, Mcmahan2013}.
This kind of algorithms do not require to load training dataset into RAM and can access it sequentially (i.e. reading from disk).
Balakrishnan and Madighan \cite{Balakrishnan2007}, Langford et al. \cite{Langford2009} report that online learning performs well when compared to batch counterparts (BBR and LASSO).

Nowadays we see the growing  number of problems where both the number of examples and the number of features are very large.
Many problems grow beyond the capabilities of a single machine and need to be handled by distributed systems.
Distributed machine learning is now an area of active research.
Efficient computational architectures and optimizations techniques allow to find more precise solutions, process larger training dataset (without subsampling), and reduce the computational burden.

Approaches to distributed training of GLMs naturally fall into two groups by the way they split data across computing nodes: by examples \cite{Agarwal2011}
or by features \cite{Peng2013}.
We believe that algorithms that split data by features can achieve better performance and faster training speed than those that split by examples.
Our experiments so far confirm that belief.

When splitting data by examples, online learning comes in handy. A model is trained in online fashion on each subset,
then parameters of are averaged and used as a warmstart for the next iteration, and so on \cite{Agarwal2011, Zinkevich2010}.
The L-BFGS and conjugate gradient methods can be easily implemented for example-wise splitting \cite{Agarwal2011}.
The log-likelihood and its gradient are separable over examples. Thus they can be calculated in parallel on parts of training set and then summed up.

Parallel block-coordinate descent is a natural algorithmic framework if we choose to split by features. The challenge here is how to combine steps from coordinate blocks, or computing nodes,
and how to organize communication. When features are independent, parallel updates can be combined straightforwardly, otherwise they may come into conflict and not yield enough improvement to objective; this has been clearly illustrated in \cite{Bradley2011}.
Bradley et al. \cite{Bradley2011} proposed Shotgun algorithm based on randomized coordinate descent.
They studied how many variables can be updated in parallel to guarantee convergence.
Ho et al. \cite{Ho2013} presented distributed implementation of this algorithm compatible with State Synchronous Parallel Parameter Server.
Richt\'{a}rik and Tak\'a\v{c} \cite{Richtarik2012} use randomized block-coordinate descent and also exploit partial separability of the objective. The latter relies on sparsity in data, which  is indeed characteristic to many large scale problems. They present theoretical estimates of speed-up factor of parallelization.
Peng et al. \cite{Peng2013} proposed a greedy block-coordinate descent method, which selects the next coordinate to update based on the estimate of the expected
improvement in the objective.
They found their GRock algorithm to be superior over parallel FISTA \cite{Beck2009} and ADMM \cite{Boyd2010}.
Smith et al. \cite{Smith2015} introduce aggregation parameter, which controls the
level of adding versus averaging of partial solutions of all machines.
Meier et al. \cite{Meier2008} used block-coordinate descent for fitting logistic regression with the group lasso penalty. They use a diagonal approximation of a Hessian and make steps over blocks of variables in a group followed by a line search to ensure convergence. All groups are processed sequentially; parallel version of the algorithm is not studied there, though the authors mention this possibility.

In contrast, our approach is to make parallel steps on all blocks, then use combined update as a direction and perform a line search.
We show that sufficient data for the line search have the size $O(n)$, where $n$ is the number of training examples,
so it can be performed on one machine.
Consequently, that's the amount of data sufficient for communication between machines.
Overall, our algorithm fits into the framework of CGD method proposed by Tseng and Yun \cite{Tseng2007}, which allows us to prove convergence.

Our main contributions are the following:
\begin{itemize}
\item We propose a new parallel coordinate descent algorithm for $L_1$ and $L_2$ regularized GLMs (Section \ref{sec:math}) and prove its convergence for linear, logistic and probit regression (Section \ref{sec:convergence})
\item We demonstrate how to guarantee sparsity of the solution by means of trust-region updates (Section \ref{sec:large_mu})
\item We develop a computationally efficient software architecture for fitting GLMs with regularizers in the distributed settings (Section \ref{software-implementation})
\item We show how our algorithm can be modified to solve the ``slow node problem'' which is common in distributed machine learning (Section \ref{sec:alb})
\item We empirically show effectiveness of our algorithm and its implementation in comparison with several state-of-the art methods for the particular case of logistic regression (Section \ref{sec:numerical})
\end{itemize}


The C++ implementation of our algorithm, which we call \texttt{d-GLMNET}, is publicly available at
\url{https://github.com/IlyaTrofimov/dlr}.

\section{Problem setting}
\label{software-implementation}

Training linear classification and regression leads to the optimization problem
\begin{equation}
\label{problem}
\vbeta^* = \argmin_{\vbeta \in \mathbb{R}^p} f(\vbeta),
\end{equation}
\begin{equation}
\label{target}
f(\vbeta)= L(\vbeta) + R(\vbeta).
\end{equation}
Where $L(\vbeta) = \sum_{i=1}^n \ell(y_i, \vbeta^T \vx_i)$ is the negated log-likelihood and $R(\vbeta)$ is a regularizer.
Here $y_i$ are targets, $\vx_i \in \mathbb{R}^p$ are input features, $\vbeta \in \mathbb{R}^p$ is the unknown vector of weights for input features.
We will denote by $nnz$ the number of non-zero entries in all $x_i$.
The function $\ell (y, \hat{y})$ is a example-wise loss which we assume to be convex and twice differentiable. Many statistical problems can be expressed in this form:
logistic, probit, Poisson regression, linear regression with squared loss, etc.

Some penalty $R(\vbeta)$ is often added to avoid overfitting and numerical ill-conditioning.
In this work we consider the elastic net regularizer
$$
R(\vbeta) = \lambda_1 \|\vbeta\|_1 + \frac{\lambda_2}{2} \|\vbeta\|^2.
$$


We solve the optimization problem (\ref{problem}) by means of a block coordinate descent algorithm.
The first part of the objective - $L(\vbeta)$ is convex and smooth.
The second part is a regularization term $R(\vbeta)$, which is convex but non-smooth when $\lambda_1 > 0$.
Hence one cannot use directly efficient optimization techniques like conjugate gradient method or L-BFGS which are often used for logistic regression with $L_2$-regularization.

Our algorithm is based on building local approximations to the objective (\ref{target}).
A smooth part $L(\vbeta)$ of the objective has quadratic approximation
\begin{align}
\label{IRLS}
& \sum_{i=1}^{n} \ell (y_i, (\vbeta + \Delta \vbeta)^T x_i) \approx L_q(\vbeta, \Delta \vbeta) \notag \\
& = \sum_{i=1}^{n}  \left\{ \ell (y_i, \vbeta^T x_i) + \frac{\partial \ell(y_i, \vbeta^T x_i)}{\partial \hat{y}} \Delta \vbeta^T \vx_i
+ \frac{1}{2}(\Delta \vbeta^T \vx_i ) \frac{\partial^2 \ell(y_i, \vbeta^T x_i)}{\partial \hat{y}^2} (\Delta \vbeta^T \vx_i) \right\} \notag \\
& = C(\vbeta) + \frac{1}{2} \sum_{i=1}^{n} w_i (z_i - \Delta \vbeta^T \vx_i)^2.
\end{align}
\begin{align*}
w_i = \frac{\partial^2 \ell (y_i, \vbeta^T x_i)}{\partial \hat{y}^2}, \\
z_i = - \frac{ \partial \ell (y_i, \vbeta^T x_i) / \partial \hat{y}}{ \partial^2 \ell (y_i, \vbeta^T x_i) / \partial \hat{y}^2}
\end{align*}
and $C(\vbeta)$ doesn't depend on $\Delta \vbeta$
\[
C(\vbeta) = L(\vbeta) - \frac{1}{2} \sum_{i = 1}^{n} z_i^2 w_i.
\]
The core idea of GLMNET and newGLMNET is iterative minimization of the penalized quadratic approximation to the objective
\begin{equation}
\label{quad}
\argmin_{\Delta \vbeta} \left\{ L_q(\vbeta, \Delta \vbeta) + R(\vbeta + \Delta \vbeta) \right\}
\end{equation}
via cyclic coordinate descent. This form (\ref{IRLS}) of approximation allows to make Newton updates of the vector $\vbeta$ without storing the Hessian explicitly.

The approximation (\ref{quad}) has a simple closed-form solution with respect to a single variable $\Delta \beta_j$
\begin{equation}
\label{beta-update}
\Delta \beta_j^* = \frac{T \left(\sum_{i=1}^{n} w_i x_{ij} q_i, \lambda_1 \right)}{\sum_{i=1}^{n} w_i x_{ij}^2 + \lambda_2} - \beta_j,
\end{equation}
\begin{align*}
T(x, a) = \sgn(x) \max(|x| - a, 0), \\
q_i = z_i - \Delta \vbeta^T \vx_i + (\beta_j + \Delta \beta_j) x_{ij}.
\end{align*}

\section{Parallel coordinate descent}
\label{sec:math}

In this work we introduce a novel architecture for a parallel coordinate descent in a distributed settings (multiple computational nodes).
The natural way to do it is to split training data set by features (``vertical'' splitting).
We will denote objects related to different computational nodes by upper indexes and use lower indexes for example and feature numbers.
More formally: let us split $p$ input features into $M$ disjoint sets $S^k$
$$
\bigcup\limits_{k=1}^{M} S^k = \{1, ..., p\}, \qquad S^m \cap S^k = \emptyset,\, k \neq m.
$$
Our approach is to optimize the quadratic approximation (\ref{quad}) in parallel over blocks of weights $\Delta \vbeta^m$.
The following proposition explains how this idea modifies the original GLMNET algorithm.
\begin{proposition} Optimizing the quadratic approximation (\ref{quad}) in parallel over blocks of weights $\Delta \vbeta^m$ is equivalent to optimizing the quadratic approximation to the objective
\begin{equation}
\label{quad-appr}
\argmin_{\Delta \vbeta} \left\{ L(\vbeta) + \nabla L(\vbeta)^T \Delta \vbeta + \frac{1}{2} \Delta \vbeta^T  \tilde{H}(\vbeta) \Delta \vbeta  + R(\vbeta + \Delta \vbeta) \right\}
\end{equation}
with block-diagonal $\tilde{H}(\vbeta)$ approximation of the Hessian

\begin{table}[h]
\begin{equation}
\label{tilde-h}
\centering
{
\begin{tabular}{cll}

\multicolumn{2}{c}{}
$(\tilde {H}(\vbeta))_{jl} = $ &
$\left\{
\begin{tabular}{l}
$(\nabla^2 L(\vbeta))_{jl}, \; \mbox{if} \; \exists m: \; j, l \in S^m$, \\
0, \; \mbox{otherwise.} \\
\end{tabular}
\right.$

\end{tabular}
}
\end{equation}
\end{table}

\end{proposition}
\begin{proof}

Let $\Delta \vbeta = \sum_{m=1}^M \Delta \vbeta^m$, where $\Delta \beta^m_j = 0$ if $j \notin S^m$. Then
\begin{align*}
L_q(\vbeta, \Delta \vbeta^m) & = L(\vbeta) + \nabla L(\vbeta)^T \Delta \vbeta^m + \frac{1}{2} \Delta (\vbeta^m)^T \nabla^2 L(\vbeta) \Delta \vbeta^m \\
& = L(\vbeta) + \nabla L(\vbeta)^T \Delta \vbeta^m + \frac{1}{2} \sum_{j,k \in S^m} (\nabla^2 L(\vbeta))_{jk} \Delta \beta^m_{j} \Delta \beta^m_{k}.
\end{align*}
By summing this equation over $m$
\begin{align}
\label{eq}
\sum_{m=1}^M L_q(\vbeta, \Delta \vbeta^m) &=
\sum_{m=1}^M \left( L(\vbeta) + \nabla L(\vbeta)^T \Delta \vbeta^m + \frac{1}{2} \sum_{j,k \in S^m} (\nabla^2 L(\vbeta))_{jk} \Delta \beta^m_{j} \Delta \beta^m_{k} \right) \notag \\
& = M L(\vbeta) + \nabla L(\vbeta)^T \Delta \vbeta + \frac{1}{2} \Delta \vbeta^T  \tilde{H}(\vbeta) \Delta \vbeta.
\end{align}
From the equation (\ref{eq}) and separability of the $L_1$ and $L_2$ penalties it follows that solving the problem in the equation (\ref{quad-appr})
is equivalent to solving $M$ independent sub-problems
\begin{equation}
\label{sub-problem}
\argmin_{\Delta \vbeta^m} \left\{ L_q(\vbeta, \Delta \vbeta^m) + \sum_{j \in S^m} R(\beta_j + \Delta \beta^m_j) \suchthat \Delta \beta^m_j = 0 \; \text{if} \; j \notin S^m \right\}, \quad m = 1 \myldots M
\end{equation}
and can be done in parallel over $M$ nodes.
\end{proof}

Doing parallel updates over blocks of weights is the core of the proposed \texttt{d-GLMNET} algorithm. Also it is possible to minimize a more general approximation
$$
L_q^{gen}(\vbeta, \Delta \vbeta) \eqdef L(\vbeta) + \nabla L(\vbeta)^T \Delta \vbeta + \frac{1}{2} \Delta \vbeta^T  (\mu(\tilde{H}(\vbeta) + \nu I)) \Delta \vbeta,
$$
\begin{equation}
\label{quad-appr-mod}
\argmin_{\Delta \vbeta} \left\{ L_q^{gen}(\vbeta, \Delta \vbeta) + R(\vbeta + \Delta \vbeta) \right\},
\end{equation}
where $\mu \ge 1, \nu > 0$, without storing the Hessian explicitly. The one-dimensional update rule modifies accordingly

\begin{equation}
\label{beta-update-mod}
\Delta \beta_j^* = \frac{T \left(\sum_{i=1}^{n} w_i x_{ij} r_i + \nu \beta_j, \lambda_1 \right)}{\mu \sum_{i=1}^{n} w_i x_{ij}^2 + \lambda_2 + \nu} - \beta_j,
\end{equation}
$$
r_i = z_i - \mu \Delta \vbeta^T \vx_i + \mu (\beta_j + \Delta \beta_j) x_{ij}.
$$
Applying $\mu > 1$ improves sparsity of the solution in case of $L_1$ regularization (see Section \ref{sec:large_mu}). 
Addition of $\nu I$ guarantees that matrix is positive definite, which is essential for convergence (see Section \ref{sec:convergence}).

We describe a high-level structure of \texttt{d-GLMNET} in the Algorithm \ref{alg:d-glmnet-high}.


\begin{algorithm}[t]

\caption{Overall procedure of d-GLMNET.}
\label{alg:d-glmnet-high}

\DontPrintSemicolon
\SetAlgoVlined
\SetKwInOut{Input}{Input}\SetKwInOut{Output}{Return}

\Input{training dataset, $\lambda_1$, $\lambda_2$, feature splitting $S^1,\myldots,S^M$, $\eta_1 \ge 1, \eta_2 \ge 1$.}
\BlankLine

$\vbeta \gets 0$. \;
$\mu \gets 1$. \;

\While{not converged} {
Do in parallel over $M$ nodes: \;
\quad Minimize $L_q^{gen}(\vbeta, \Delta \vbeta^m) + R(\vbeta + \Delta \vbeta^m)$ with respect to $\Delta \vbeta^m$ (Algorithm \ref{alg:d-glmnet-sub}). \label{d-glmnet-high:step} \;
$\Delta \vbeta \gets \sum_{m=1}^{M} \Delta \vbeta^m$. \;
Find $\alpha \in (0, 1]$ by the line search procedure (Algorithm \ref{alg:d-glmnet-line-search}). \label{d-glmnet-high:alpha}\;
$\vbeta \gets \vbeta + \alpha \Delta \vbeta$.  \;

\eIf{$\alpha < 1$} {
$\mu \gets \eta_1 \mu$.
}{
$\mu \gets \max(1, \mu / \eta_2)$.
}

}
\BlankLine
\Output{$\vbeta$.}
\end{algorithm}

Algorithm \ref{alg:d-glmnet-sub} presents our approach for minimization of the local approximation (\ref{quad-appr-mod}) with respect to $\Delta \vbeta^m$.
\texttt{d-GLMNET} makes one cycle of coordinate descent over input features, while GLMNET and newGLMNET use multiple passes; we found that our approach works well in practice.

\begin{algorithm}[t]

\DontPrintSemicolon
\SetAlgoVlined
\SetKwInOut{Input}{Input}\SetKwInOut{Output}{Return}


\caption{Solving quadratic sub-problem at node $m$.}
\label{alg:d-glmnet-sub}

$\Delta \vbeta^m \gets 0$. \;
\ForEach{$j \in S^m$}{
Minimize $L_q^{gen}(\vbeta, \Delta \vbeta^m) + R(\vbeta + \Delta \vbeta^m)$ with respect to $\Delta \beta^m_j$ using (\ref{beta-update-mod}). \;
}
\BlankLine
\Output{$\Delta \vbeta^m$.}

\end{algorithm}

Like in other Newton-like algorithms a line search should be done to guarantee convergence.
Algorithm \ref{alg:d-glmnet-line-search} describes our line search procedure.
We found that selecting $\alpha_{init}$ by minimizing the objective (\ref{target}) (step \ref{alg:d-glmnet-line-search-alpha-init}, Algorithm \ref{alg:d-glmnet-line-search}) speeds up the convergence of the Algorithm \ref{alg:d-glmnet-high}.
We used $b = 0.5,\, \sigma = 0.01, \, \gamma = 0$ in line search procedure for numerical experiments (Section \ref{sec:numerical}).


\begin{algorithm}[t]

\DontPrintSemicolon
\SetKwInOut{Input}{Input}\SetKwInOut{Output}{Return}


\caption{Line search procedure.}
\label{alg:d-glmnet-line-search}

\KwData{$\delta > 0, 0 < b < 1, 0 < \sigma < 1, 0 \leq \gamma < 1$.}
\BlankLine

\eIf{$\alpha = 1$ yields sufficient decrease in the objective (\ref{armijo})}
{$\alpha \gets 1$.}
{
Find $\alpha_{init} = \argmin_{\delta < \alpha \leq 1} f(\vbeta + \alpha \Delta \vbeta)$. \label{alg:d-glmnet-line-search-alpha-init} \;
Armijo rule: let $\alpha$ be the largest element of the sequence $\{\alpha_{init} b^j\}_{j=0,1,...}$ satisfying
\begin{align}
\label{armijo}
f(\vbeta + \alpha \Delta \vbeta) & \leq f(\vbeta) + \alpha \sigma D, \\
D = \nabla L(\vbeta)^T \Delta \vbeta + & \gamma \Delta \vbeta^T (\mu(\tilde{H}(\vbeta) + \nu I)) \Delta \vbeta + R(\vbeta + \Delta \vbeta) - R(\vbeta). \notag
\end{align}
}
\BlankLine
\Output{$\alpha$.}

\end{algorithm}


\section{Ensuring sparsity}
\label{sec:large_mu}

Applying $\mu > 1$ is required for providing sparse solution in case of $L_1$ regularization.
Sparsity may suffer from the line search.
Algorithm \ref{alg:d-glmnet-high} starts with $\vbeta = 0$, so absolute values of $\vbeta$ tend to increase.
However there may be cases when  $\Delta \beta_j = -\beta_j$ for some $j$ on step \ref{d-glmnet-high:step} of Algorithm \ref{alg:d-glmnet-high}, so $\beta_j$ can go back to $0$.
In that case, if line search on step \ref{d-glmnet-high:alpha} selects $\alpha < 1$,  then the opportunity for sparsity is lost.
Parallel steps over blocks of weights $\Delta \vbeta^m$ come in conflict and for some datasets Algorithm \ref{alg:d-glmnet-line-search} selects $\alpha < 1$ almost always.

To guarantee sparsity of the solution an algorithm must select step size $\alpha = 1$ often enough.
In Appendix \ref{app:large_mu} we prove that when $\mu \ge \frac{\Lambda_{max}}{(1 - \sigma)\lambda_{min}}$ the line search is not required at all. Here $\Lambda_{max}, \lambda_{min}$ are maximal and minimal eigenvalues of $H(\vbeta)$ and $\tilde{H}(\vbeta)$ respectively.
However it is hard to compute $\Lambda_{max}, \lambda_{min}$ for an arbitrary dataset. Also $\mu \ge \frac{\Lambda_{max}}{(1 - \sigma)\lambda_{min}}$ may yield very small steps and slow convergence.
For his reason in \texttt{d-GLMNET} the $\mu$ parameter is changed adaptively, see Algorithm \ref{alg:d-glmnet-high}.
In numerical experiments we used $\eta_1 = \eta_2 = 2$.

%
%
%
%
%
%

Note that line search always yields $\alpha = 1$ when $\mu \ge \frac{\Lambda_{max}}{(1 - \sigma)\lambda_{min}}$, thus adaptive algorithm preserves $1 \le \mu < \frac{\eta_1
 \Lambda_{max}}{(1 - \sigma)\lambda_{min}}$.

Alternatively problem (\ref{quad-appr-mod}) with $\mu > 1$ can be interpreted as minimization of the Lagrangian for the constrained optimization problem
\begin{gather*}
\argmin_{\Delta \vbeta} \left\{ L(\vbeta) + \nabla L(\vbeta)^T \Delta \vbeta + \frac{1}{2} \Delta \vbeta^T (\tilde{H}(\vbeta) + \nu I) \Delta \vbeta  + R(\vbeta + \Delta \vbeta) \right\} \\
\text{subject to:} \quad \Delta \vbeta^T (\tilde{H}(\vbeta) + \nu I) \Delta \vbeta \le r,
\end{gather*}
%
with Lagrange multiplier $\mu - 1$.
Steps are constrained to the iteration specific trust region radius $r$.


\section{Convergence}
\label{sec:convergence}

Algorithm \texttt{d-GLMNET} falls into the general framework of block-coordinate gradient descent (CGD) proposed by Tseng and Yun \cite{Tseng2007}, which we briefly describe here.
CGD is about minimization of a sum of a smooth function and separable convex function  (\ref{target});  in our case, negated log-likelihood and elastic net penalty.
At each iteration CGD solves penalized quadratic approximation problem
\begin{align}
\label{cgd-quad}
\argmin_{\Delta \vbeta} \left\{ L(\vbeta) + \nabla L(\vbeta)^T \Delta \vbeta + \frac{1}{2} \Delta \vbeta^T H \Delta \vbeta + R(\vbeta + \Delta \vbeta) \right\},
\end{align}
where  $H$ is some positive definite matrix, possibly iteration specific. For convergence it also requires that for some $e_{min}, e_{max} > 0$ for all iterations
\begin{equation}
\label{assumption}
e_{min} I \preceq H \preceq e_{max} I.
\end{equation}
At each iteration updates are done over some subset of weights.
After that a line search by the Armijo rule should be conducted.
If all weights are updated every $T \ge 1$ consecutive iterations then CGD converges globally.
Tseng and Yun \cite{Tseng2007} prove also that if $L(\vbeta)$ is strictly convex and weights are updated by a proper schedule (particularly by updating all weights at each iteration, which is our case) then $f(\vbeta)$ converges as least Q-linearly and $\vbeta$ converges at least R-linearly,

Firstly, let us show for which loss functions (\ref{assumption}) holds.
From loss function convexity follows $0 \preceq \nabla^2 L(\vbeta)$.
If the second derivative of loss function is bounded
\begin{equation}
\label{second_der}
\frac{\partial^2 \ell(y, \hat{y})}{\partial \hat{y}^2} < M,
\end{equation}
then
\begin{align*}
\va^T \nabla^2 L(\vbeta) \va & = \sum_{i=1}^n (\va^T \vx_i ) \frac{\partial^2 \ell(y_i, \vbeta^T x_i)}{\partial \hat{y}^2} (\va^T \vx_i)
\le \|\va \|^2 \sum_{i=1}^n \| \vx_i \|^2 M,
\end{align*}
and we conclude that for some $\Lambda_{min}, \Lambda_{max} > 0$
\begin{equation}
\label{full-hessian-limit}
\Lambda_{min} I \preceq \nabla^2 L(\vbeta) + \nu I \preceq \Lambda_{max} I.
\end{equation}
The assumption (\ref{second_der}) holds for logistic and probit regressions, and for linear regression with squared loss, see Appendix \ref{app:second_der}.


Secondly, let us prove (\ref{assumption}) for a block-diagonal approximation  $H = \mu (\tilde{H}(\vbeta) + \nu I)$, where $\tilde{H}(\vbeta)$ is defined in (\ref{tilde-h}).
Denote its diagonal blocks by $H^1,...,H^M$
and represent an arbitrary vector $\va$ as a concatenation of subvectors of corresponding size:
$\va = ((\va^1)^T,...,(\va^M)^T)^T$.
Then we have
$$
\va^T H \va = \sum_{m=1}^{M} (\va^m)^T H^m \va^m.
$$
Notice that $H^m = \mu (\nabla^2 L(\vbeta^m) + \nu I)$, where $\nabla^2 L(\vbeta^m)$ is a Hessian over the subset of features $S^m$.
So for each $(\nabla^2 L(\vbeta^m) + \nu I)$ we have
\begin{equation*}
\Lambda^m_{min} I \preceq \nabla^2 L(\vbeta^m) + \nu I \preceq \Lambda^m_{max} I, \quad \text{for} \; m=1, \myldots ,M.
\end{equation*}
That means
\begin{equation}
\label{block-limits}
\mu \Lambda^m_{min} \|\vx^m\|^2 \leq (\va^m)^T H^m \va^m \leq \mu \Lambda^m_{max} \|\va^m\|^2, \quad \text{for} \; m=1, \myldots ,M.
\end{equation}
Let
$$\lambda_{min} = \min_{m=1, \myldots ,M} \Lambda^m_{min}, \qquad \lambda_{max} = \max_{m=1, \myldots ,M} \Lambda^m_{max}.$$
Assume $1 \le \mu \le \mu_{max}$, which holds for constant or adaptively changing $\mu$ (Section \ref{sec:large_mu}), and by summing (\ref{block-limits}) up over $m$ we obtain the required
\begin{gather*}
\lambda_{min} \|\va\|^2 \leq \va^T H \va \leq \mu_{max} \lambda_{max} \|\va\|^2.
\end{gather*}

\section{Architecture for distributed training}
\label{software-implementation}

In this section we describe details or software implementation of parallel coordinate descent.
This implementation works in a distributed settings (multiple computational nodes).
When parallel coordinate descent is concerned the natural way is to split training data set by features (``vertical'' splitting).
Splitting of input features leads to splitting the matrix $X$ of features into $M$ parts $X^m$.
The node $m$ stores the part $X^m$ of training dataset corresponding to a subset $S^m$ of input features. Let
$$
\vbeta = (  (\vbeta^1)^T, \myldots, (\vbeta^M)^T)^T.
$$
Algorithm \ref{alg:program} presents a high-level structure of our approach.
Each node makes a step $\Delta \vbeta^m$  over its block of variables $S^m$.
Then all these steps are summed up and multiplied by a proper step size multiplier.

\begin{algorithm}[t]

\caption{Distributed coordinate descent.}
\label{alg:program}

\DontPrintSemicolon
\SetAlgoVlined
\SetKwInOut{Input}{Input}\SetKwInOut{Output}{Return}

\Input{training dataset, $\lambda_1$, $\lambda_2$, feature splitting $S^1,\myldots,S^M$.}
\BlankLine

\While {not converged}{
\quad Do in parallel over $M$ nodes: \;
\quad \quad Read part of training dataset $X^m$ sequentially. \;
\quad \quad Find updates $\Delta \vbeta^m$ and $X^m \Delta \vbeta^m$ for weights in $P^m \subseteq S^m$. \label{alg-step:beta-update} \;
\quad \quad Sum up vectors $X^m \Delta \vbeta^m $ using MPI\_AllReduce: \;
\quad \quad \quad $X \Delta \vbeta \gets \sum_{m=1}^{M} X^m \Delta \vbeta^m$. \label{alg-step:delta-sum} \;
\quad \quad Find step size $\alpha$ using line search (Algorithm \ref{alg:d-glmnet-line-search}). \label{alg-step:linear-search} \;
\quad \quad $\vbeta^m \gets \vbeta^m + \alpha \Delta \vbeta^m$.  \;
\quad \quad $X \vbeta \gets X \vbeta + \alpha X \Delta \vbeta$. \;
}
\BlankLine
\Output{$\vbeta$.}

\end{algorithm}

Algorithm \ref{alg:program} has several key features:
\begin{enumerate}

\item The weights vector $\vbeta$ is stored in the distributed manner across all nodes; a node $m$ stores $\vbeta^m$.

\item The program stores in RAM only vectors $\vy$, $X \vbeta$ \footnote{For a particular case of logistic regression one can store vector $\exp(X \vbeta)$ instead of $X \vbeta$ to speed-up computations.}, $X \Delta \vbeta$, $\vbeta^m$, $\Delta \vbeta^m$.

Thus the memory footprint at node $m$ is $3n + 2|S^m|$.

\item The program maintains vector $X \vbeta$ synchronized across all nodes after each iteration. Synchronization is done by means of summation $ X^m \Delta\vbeta^m$ on step \ref{alg-step:delta-sum}. The total communication cost is $Mn$.

\item At each iteration a subset $P^m \subseteq S^m $ of weights is updated. In Section \ref{sec:alb} we describe two subset selection strategies.

\item Various types of coordinate-wise updates can be implemented on step \ref{alg-step:beta-update}.
Our update is described in Section \ref{sec:math}.
This is done by one pass over the training dataset part $X^m$.

\item
The program reads training dataset sequentially from disk instead of RAM. It may slow down the program in case of smaller datasets, but it makes the program more scalable.
Also it conforms to the typical pattern of a multi-user Map/Reduce cluster system: large disks, many jobs started by different users are running simultaneously.
Each job might process large data but it is allowed to use only a small part of RAM at each node.

\item Doing a linear search on step \ref{alg-step:linear-search} requires calculating the log-likelihood $L(\vbeta + \alpha \Delta \vbeta)$ and $R(\vbeta + \alpha \Delta \vbeta)$ for arbitrary $\alpha \in (0, 1]$. Since the vector $X \vbeta$ is synchronized between nodes, the log-likelihood can be easily calculated.
Each node calculates the regularizer $R(\vbeta^m)$ separately and then the values are summed up via MPI\_AllReduce \footnote{We used an implementation from the Vowpal Wabbit project\\  $https://github.com/JohnLangford/vowpal\_wabbit$}.
This could be done for separable regularizers like $L_1, L_2$, group lasso, SCAD, e.t.c.

\item We use Single Program Multiple Data (SPMD) programming style, i.e., all computational nodes execute the same code and there's no selected master.
 Some operations (like linear search on step \ref{alg-step:linear-search}) are redundantly executed on all nodes.
\end{enumerate}

Typically most datasets are stored in ``by example'' form, so a transformation to ``by feature'' form is required for distributed coordinate descent.
For large datasets this operation is hard to do on a single node.
We use a Map/Reduce cluster \cite{Dean2004} for this purpose.
Partitioning of the training dataset over nodes is done by means of a Reduce operation in the streaming mode.
We did not implement parallel coordinate descent completely in the Map/Reduce programming model since it is ill-suited for iterative machine learning algorithms \cite{Low2010, Agarwal2011}.
Using other programming models like Spark \cite{Zaharia2010} looks promising.


\section{Adaptive selection of subset to update}
\label{sec:alb}
Algorithm \ref{alg:program} is flexible in selecting the subset of weights $P^m \subseteq S^m$, which are updated on step \ref{alg-step:beta-update}.
The simplest strategy is to update always all weights $P^m = S^m$. In this case the algorithm becomes deterministic; each node performs predefined computation before synchronization.
This is a case of a Bulk Synchronous Parallel (BSP) programming model. Programs based on BSP model share a common weak point: the whole program performance is limited by the slowest worker because of the synchronization step. Performance characteristics of nodes in a cluster may be different due to several reasons: competition for resources with other tasks, different hardware, bugs in optimization settings e.t.c \cite{Dean2004}.
In context of distributed data processing this problem is known as a ``slow node problem''.

In Map/Reduce clusters the slow node problem is typically solved by ``backup tasks'' \cite{Dean2004}, when the scheduler starts copies of the slowest task in a job on alternative nodes.
The completion of any copy is sufficient for the whole job completion.
However this mechanism isn't applicable for algorithms maintaining state, which is our case:
algorithm \texttt{d-GLMNET} maintains state as vectors $X \vbeta, \vbeta^m$.

Many machine learning systems try to overcome this weak point by moving to asynchronous computations. Ho et al. \cite{Ho2013} develop a Stale Synchronous Parallel Parameter Server (SSPPS), which allows the fastest and the slowest node to have some gap not exceeding a fixed number of iterations. The Fugue system \cite{Kumar2014} allows fast nodes to make extra optimization on its subset of the training dataset while waiting for slow ones.
The Y!LDA \cite{smola2010architecture} architecture for learning topic models on the large scale keeps global state in a parameter server, and each node updates it asynchronously.

Our program resolves ``slow node'' problem by selecting the subset $P^m$ adaptively. The program has an additional thread checking how many nodes have already done update over all weights in $S^m$ (Algorithm \ref{alg:program}, step \ref{alg-step:beta-update}). If the fraction of such nodes is greater then $\kappa M$ with some $0 < \kappa < 1$, then all nodes break the optimization and proceed to synchronization (Algorithm \ref{alg:program}, step \ref{alg-step:delta-sum}). Updates of weights in $S^m$ are done cyclically, so on the next iteration a node resumes optimization starting from the next weight in $S^m$. Fast nodes are allowed to make more then one cycle of optimization, i.e. make two or more updates of each weight. We used $\kappa = 0.75$ in all numerical experiments.

We call this mechanism ``Asynchronous Load Balancing'' (ALB) and the algorithm modified in this way - \texttt{d-GLMNET-ALB}.
A possible drawback of this mechanism is that a very slow node may not be able to update every weight even after many iterations. However, we have never observed such an extreme situation in practice.

\texttt{d-GLMNET-ALB} algorithm converges globally, just like \texttt{d-GLMNET}.
However, linear convergence cannot be established because \texttt{d-GLMNET-ALB} updates blocks of weights in an non-deterministic order, so it doesn't fit schedule requirements specified in \citep{Tseng2007}.

\section{Numerical experiments}
\label{sec:numerical}

In this section we evaluate the performance of \texttt{d-GLMNET} for an important particular case: logistic regression with $L_1$ and $L_2$ regularization.

Firstly, we demonstrate the effect of adaptive $\mu$ in the Hessian approximation. This strategy was used only for experiments with $L_1$ regularization.
With $L_2$ regularization we used constant $\mu = 1$.
Secondly, we perform comparison with three state-of-the-art approaches:  ADMM, distributed ``online learning via truncated gradient'' (for $L_1$ regularization), combination of distributed online learning with L-BFGS (for $L_2$ regularization). We briefly describe these approaches below. For $L_1$ regularization, the sparsity of the solution is another matter of interest. We also evaluated the effect of the ``Asynchronous Load Balancing'' (ALB) technique.
Thirdly, we show how the performance of \texttt{d-GLMNET-ALB} improves with the increase of number of computing nodes.

%

\subsection{Competing algorithms}

The first approach for comparison is an adaptation of ADMM for $L_1$-regularized logistic regression. We implemented the algorithm from \cite[sections 8.3.1, 8.3.3]{Boyd2010}\footnote{The update rule for $\bar{z}^k$ in \cite[section 8.3.3]{Boyd2010} has an error. Instead of $(\rho/2)$ should be $(\rho N/2)$. The ADMM algorithm performed poorly before we fixed it.
}. It uses a sharing technique \cite[section 7.3]{Boyd2010} to distribute computations among nodes. The sharing technique requires dataset being split by features.
Like our algorithm, it stores weights $\vbeta$ in a distributed manner. We used MPI\_AllReduce to sum up $Ax^k$ and implemented a lookup-table proposed in \cite[section 8.3.3]{Boyd2010} to speed up \mbox{$z-$update}.
Doing an \mbox{$x-$update} involves solving a large scale LASSO. We used a Shooting \cite{Fu1998} to do it since it is well suited for large and sparse datasets.
Shooting algorithm is based on coordinate descent. That is why this modification of ADMM can be viewed as another way to do distributed coordinate descent.
For each dataset we selected a parameter $\rho \in [4^{-3}, \myldots, 4^{3}]$ yielding best objective after 10 iterations and used it for final performance evaluation.


A combination of distributed online learning with L-BFGS was presented in \cite{Agarwal2011}.
An Algorithm 2 from \cite{Agarwal2011} describes a whole combined approach.
The first part of it proposes to compute a weighted average of classifiers trained at $M$ nodes independently via online learning.
The second part warmstarts L-BFGS with the result of the first part.
This combination has fast initial convergence (due to online learning) and fast local convergence (due to quasi-Newtonian updates of L-BFGS).

As we pointed out earlier, L-BFGS it not applicable for solving logistic regression with $L_1$-regularization.
Thus for experiments with $L_1$ regularization we used only distributed online learning, namely ``online learning via truncated gradient'' \cite{Langford2009}
.

For experiments with $L_2$ regularization we ran full Algorithm 2 from \cite{Agarwal2011}.
Both of these algorithms require training dataset partitioning by examples over $M$ nodes.
We used the online learning and L-BFGS implementation from the open source \texttt{Vowpal Wabbit (VW)} project\footnote{\url{https://github.com/JohnLangford/vowpal_wabbit}, version 7.5}.
We didn't use feature hashing since it may decrease the quality of the classifier.
As far as hyperparameters for online learning are concerned, we tested jointly learning rates (raging from $0.1$ to $0.5$) and powers of learning rate decay (raging from 0.5 to 0.9).
Then we selected the best combination for each dataset (yielding the best objective) and used it for further tests.

We would like to note that the \texttt{d-GLMNET} and \texttt{d-GLMNET-ALB} don't have any hyperparameters (except a regularization coefficient) and they are easier for practical usage.

\subsection{Datasets and experimental settings}
\begin{table}
\caption{Datasets summary.}
{
\begin{tabular}{|c|c|c|c|c|c|}
  \hline
  dataset & size & \#examples (train/test/validation) & \#features & nnz & avg nonzeros \\
  \colrule
  epsilon     &  12 Gb & $0.4 \times 10^6$ / $0.05 \times 10^6$ / $0.05 \times 10^6$  & $2000$               & $8.0 \times 10^8$  & 2000 \\
  webspam     &  21 Gb & $0.315 \times 10^6$ / $0.0175 \times 10^6$ / $0.0175 \times 10^6$ & $16.6 \times 10^6$ & $1.2 \times 10^9$  & 3727 \\
  yandex\_ad & 56 Gb & $57 \times 10^6$ / $2.35 \times 10^6$  / $2.35 \times 10^6$ & $35 \times 10^6$              & $5.7 \times 10^9$  & 100 \\
  \hline
\end{tabular}
}
\label{tbl:datasets}
\end{table}

We used three datasets for numerical experiments:

\begin{itemize}
\item \textbf{epsilon} - Synthetic dataset.
\item \textbf{webspam} - Webspam classification problem.
\item \textbf{yandex\_ad} - The click prediction problem - the goal is to predict the probability of click on the ad. This is a non-public dataset created from the user logs of the commercial search engine (Yandex).
\end{itemize}

Two of these sets - ``epsilon'' and ``webspam'' are publicly available from the Pascal Large Scale Learning Challenge 2008 \footnote{\url{http://largescale.ml.tu-berlin.de/}. We used preprocessing and train/test splitting from \\
\url{http://www.csie.ntu.edu.tw/~cjlin/libsvmtools/datasets/binary.html}. \\}.
We randomly split the original test sets into new test and validation sets.

The datasets are summarized in the Table \ref{tbl:datasets}.
Numerical experiments were carried out at cluster of multicore blade servers having Intel(R) Xeon(R) CPU E5-2660 2.20GHz, 32 GB RAM, connected by Gigabit Ethernet.
We used 16 nodes of a cluster in all numerical experiments.
Each node ran one instance of the algorithms.

We used Map/Reduce cluster to partition dataset by features over nodes. This was done by a Reduce operation in streaming mode using feature number as a key. Since Reduce operation assigns partitions to nodes by a hash of a key, the splitting of input features $S^1, \myldots, S^M$ was pseudo-random.
All algorithms requiring dataset partitioning by feature used the same partitioning.

Table \ref{tbl:performance} presents computational load on all nodes for each of the algorithms.

\begin{table}
\caption{Computational load of the algorithms.}

\begin{tabular}{|c|c|c|c|}
  \hline
  Algorithm & Iteration Complexity & Memory Footprint & Communication Cost \\
  \hline
  \begin{tabular}{c} Online learning \\ via truncated gradient \end{tabular} & $O(nnz)$ & $2Mp $ & $2Mp$ \\
  L-BFGS & $O(nnz)$ & $ 2rMp \footnote{The $r$ parameter specifies memory usage in L-BFGS. We used default value $r=15$.} $ & $ Mp $ \\
  d-GLMNET &  $O(nnz)$ & $3Mn + 2p$ & $Mn$\\
  ADMM & $O(nnz)$ & $5Mn + p$ & $ Mn$ \\
 \hline
\end{tabular}

\label{tbl:performance}
\end{table}

For each dataset we selected $L_1$ and $L_2$ regularization coefficients from the range $\{2^{-6}, \ldots, 2^{6}\}$ yielding the best classification quality on the validation set.
Then we ran the software implementations of all algorithms (\texttt{d-GLMNET}, ADMM, distributed ``online learning via truncated gradient'', L-BFGS) with the best hyperparameters
on all datasets.


To make evaluation less dependent on the current situation on the cluster, we repeated learning 9 times with each algorithm and selected a run with the median execution time. To study convergence profile of each algorithm we recorded the relative objective suboptimality and testing quality versus time. The optimal value of the objective function $f^*$ was approximately evaluated by running many iterations of \texttt{liblinear}\footnote{\url{http://www.csie.ntu.edu.tw/~cjlin/liblinear/}} program for ``epsilon'' and ``webspam'' datasets and \texttt{d-GLMNET} for the biggest ``yandex\_ad'' dataset.
Then relative objective suboptimality was calculated as $(f - f^* ) / f^*$, where $f$ is the current value of the objective function.
We used area under precision-recall curve (auPRC) as a testing quality measure (the definition is given in Appendix \ref{sec:auprc}).

Finally, we evaluated the influence of the number of computing nodes on the speed of \texttt{d-GLMNET-ALB}. Fig. \ref{fig:l1-speedup} and \ref{fig:l2-speedup} present execution times of the algorithm with various numbers of nodes relative to one node. Time was recorded when the algorithm came within $2.5\%$ of the optimal objective function value $f^*$.

\begin{figure}[t]
\begin{center}
        \subfigure[relative objective suboptimality vs. time] {
                \includegraphics[width=0.31\textwidth]{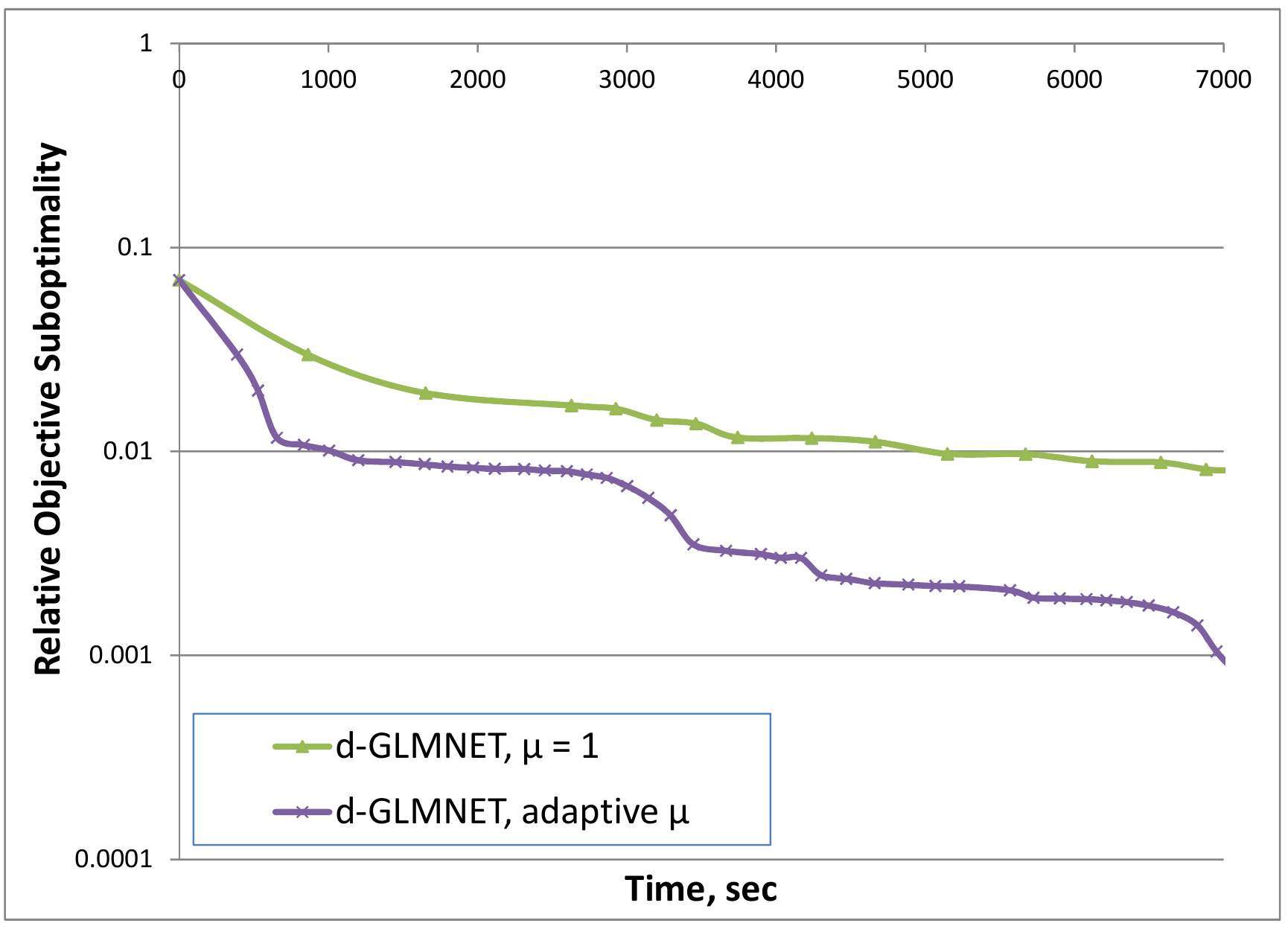}
        }
        \subfigure[testing quality (area under precision-recall curve) vs. time] {
                \includegraphics[width=0.31\textwidth]{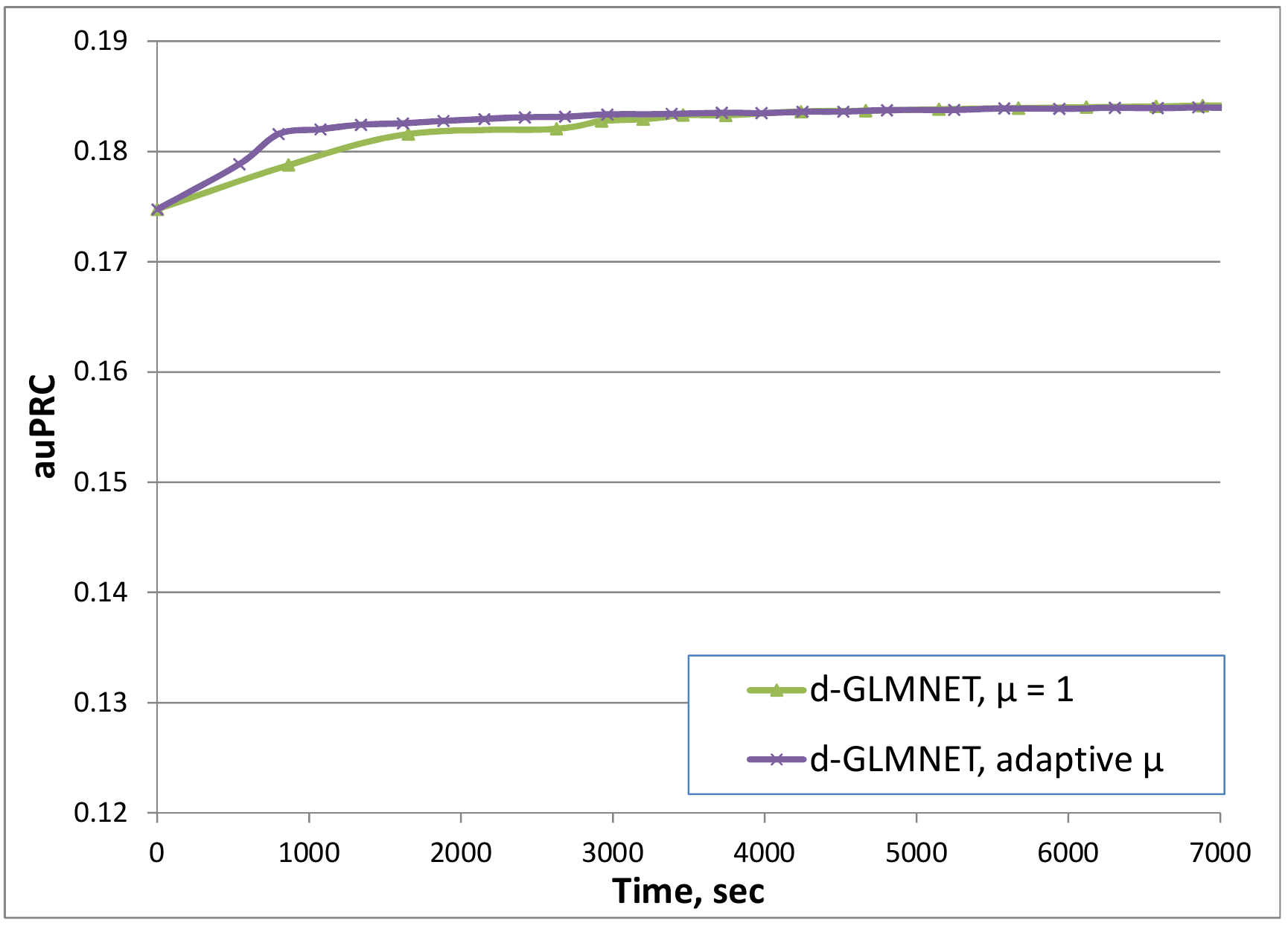}
        }
        \subfigure[number of non-zero weights vs. time] {
                \includegraphics[width=0.31\textwidth]{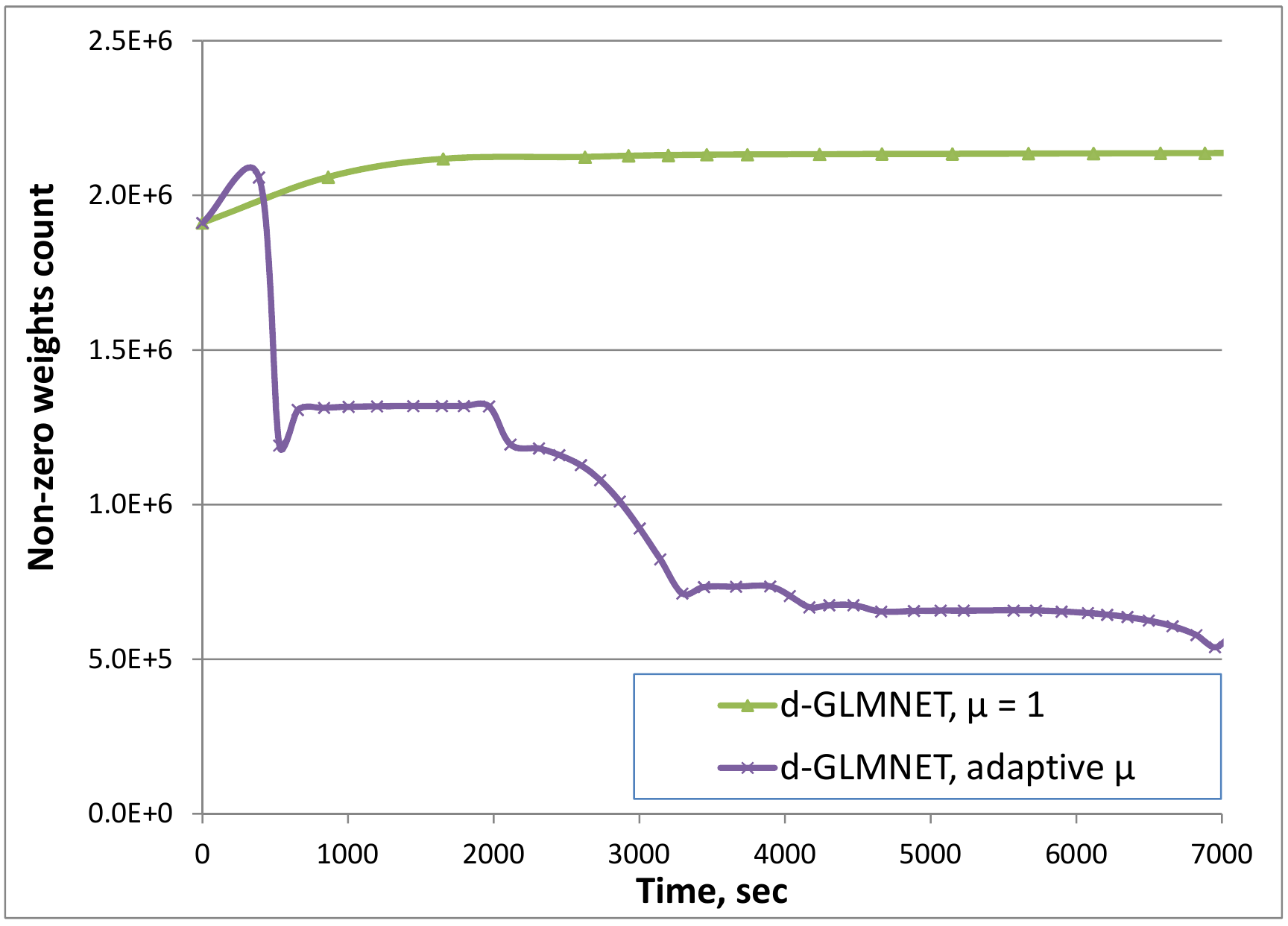}
        }
        \caption{Constant $\mu = 1$ vs. adaptive $\mu$ for yandex\_ad dataset, $L_1$ regularization.}
        \label{fig:yandex_ad-mu}
\end{center}
\end{figure}

\begin{figure}[t]
\begin{center}
        \subfigure[webspam] {
                \includegraphics[width=0.31\textwidth]{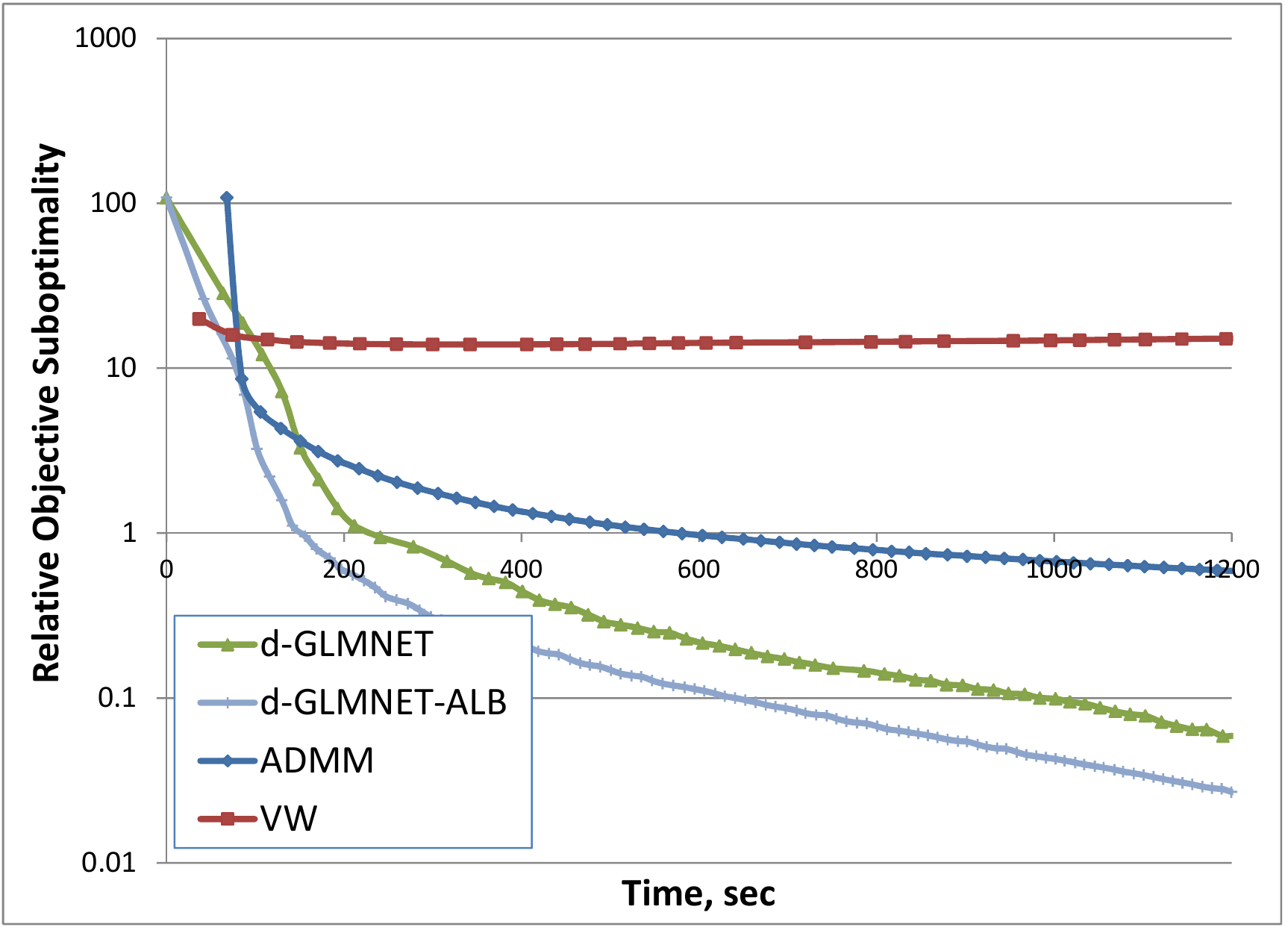}
        }
        \subfigure[yandex\_ad] {
                \includegraphics[width=0.31\textwidth]{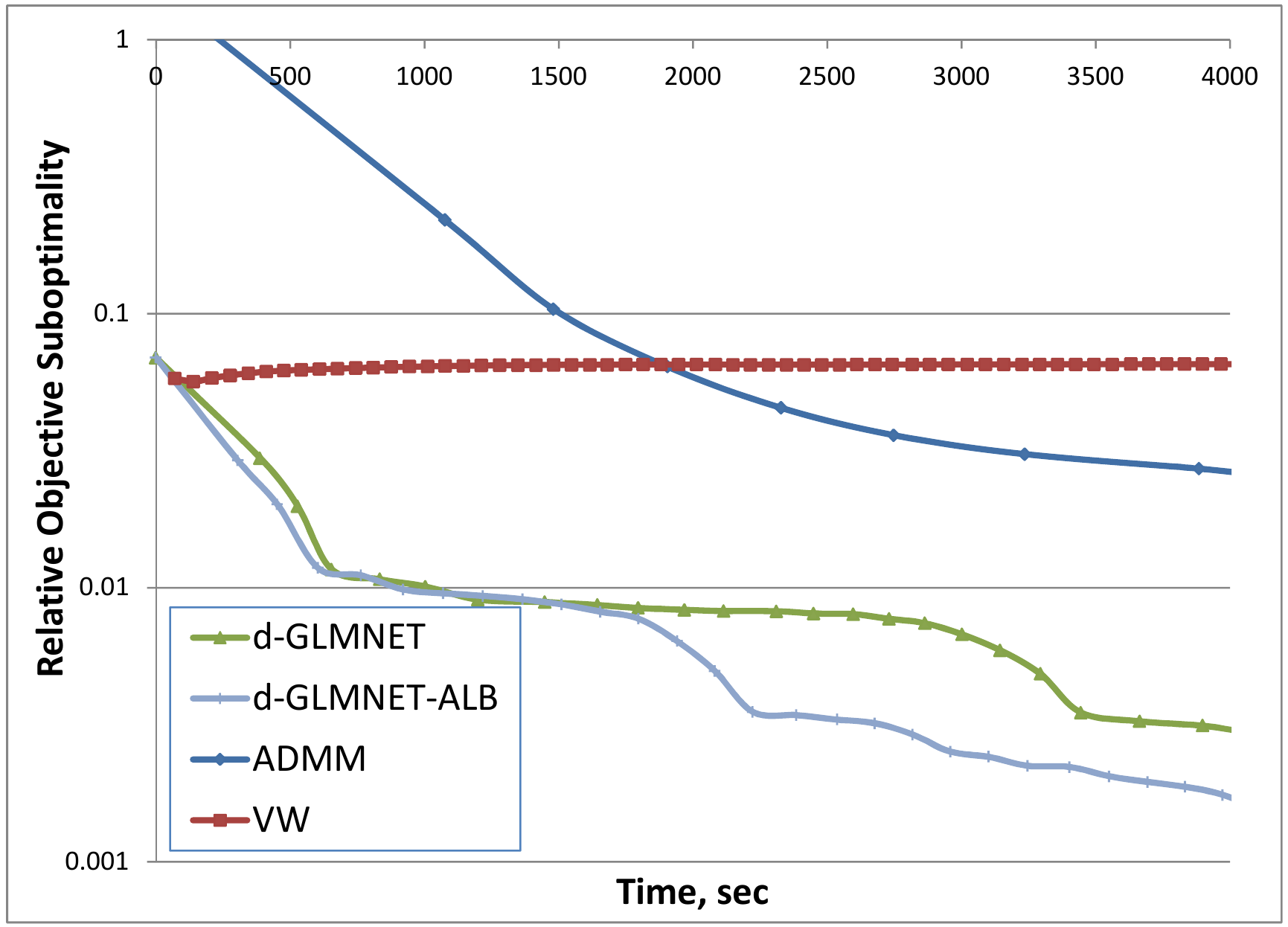}
        }
        \subfigure[epsilon] {
                \includegraphics[width=0.31\textwidth]{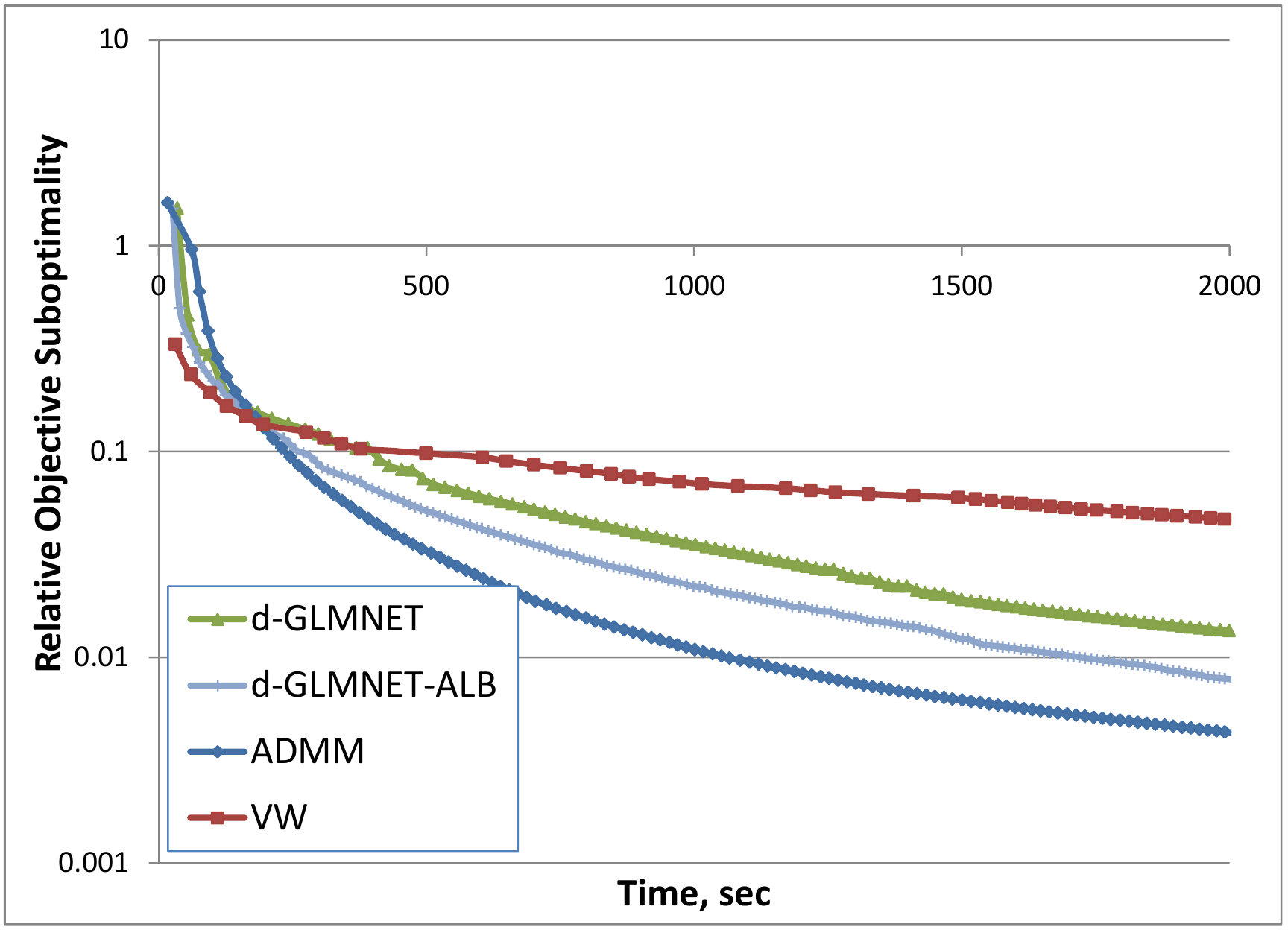}
        }
        \caption{$L_1$ regularization: relative objective suboptimality vs. time.}
        \label{fig:l1-target}
\end{center}
\begin{center}
        \subfigure[webspam] {
                \includegraphics[width=0.31\textwidth]{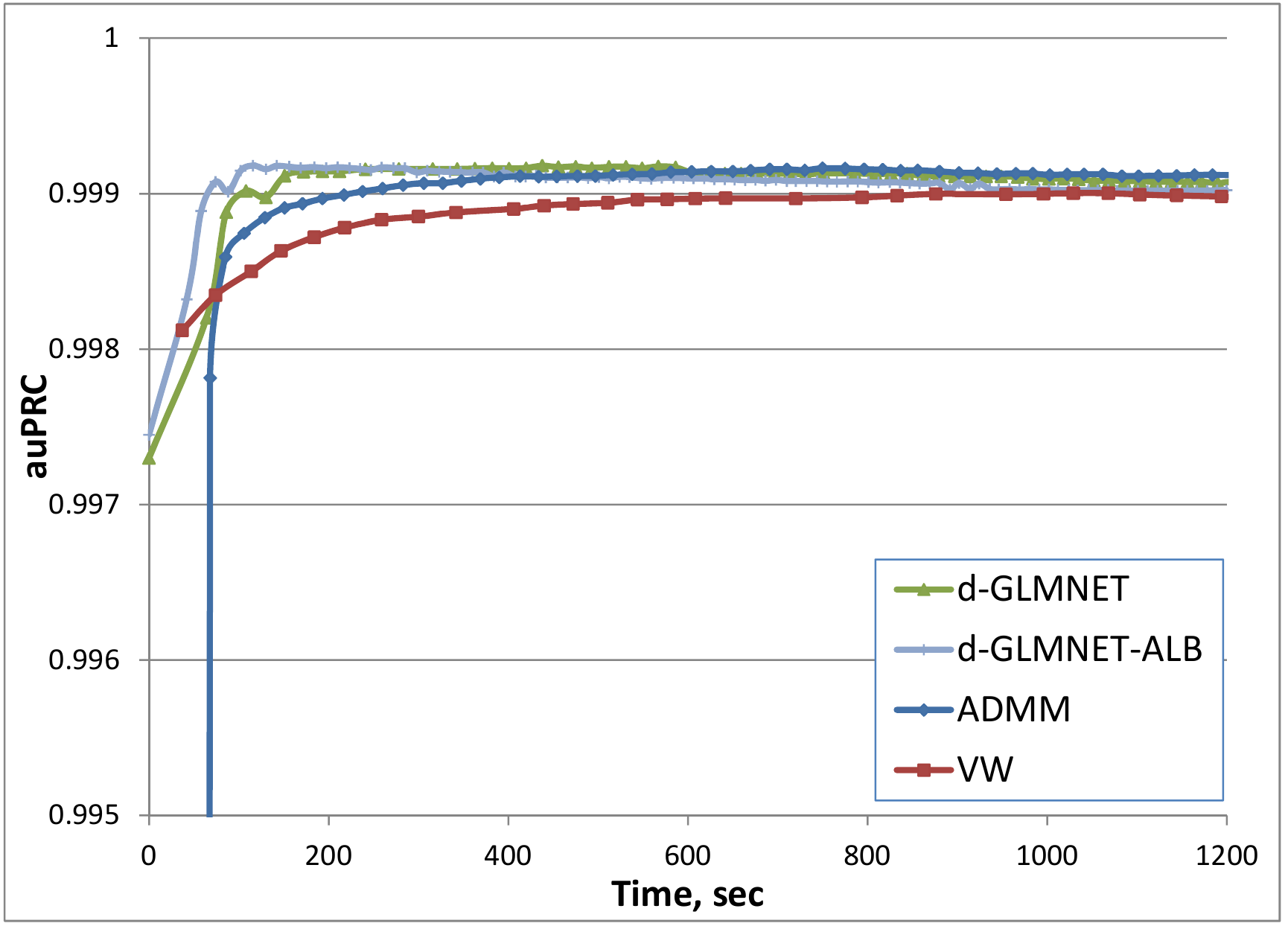}
        }
        \subfigure[yandex\_ad] {
                \includegraphics[width=0.31\textwidth]{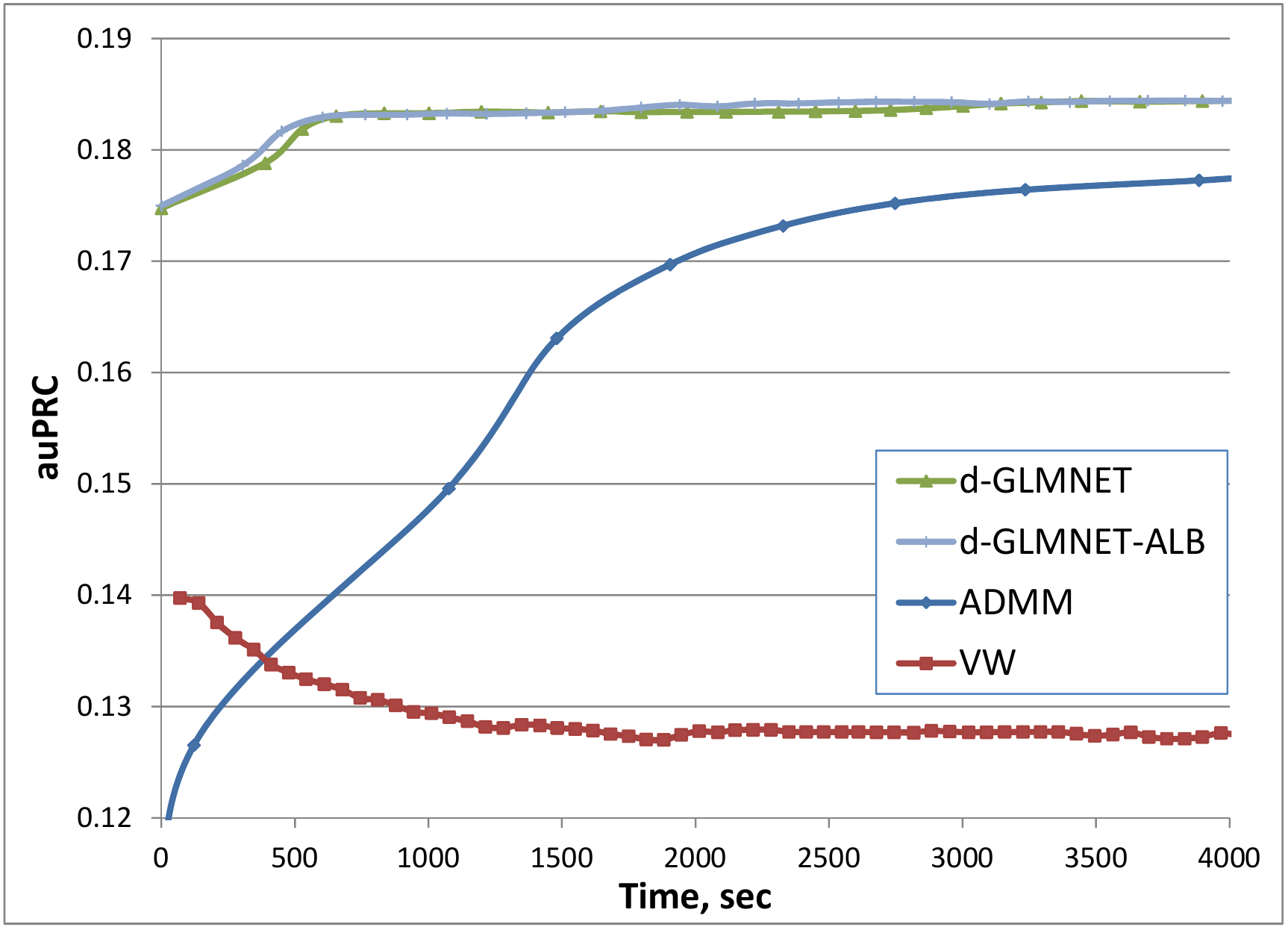}
        }
        \subfigure[epsilon] {
                \includegraphics[width=0.31\textwidth]{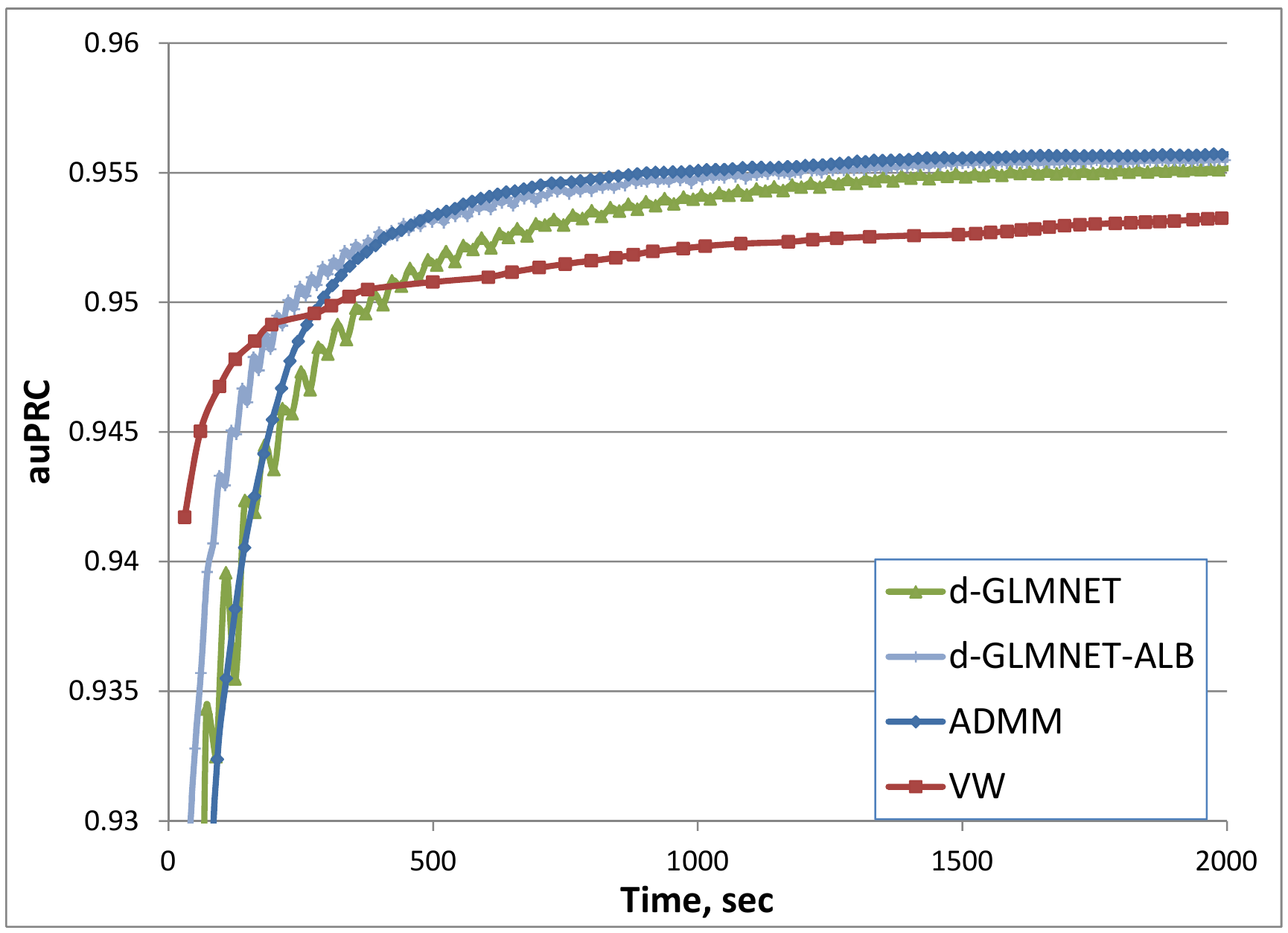}
        }
       \caption{$L_1$ regularization: testing quality (area under precision-recall curve) vs. time.}
       \label{fig:l1-testing}
\end{center}
\begin{center}
        \subfigure[webspam] {
                \includegraphics[width=0.31\textwidth]{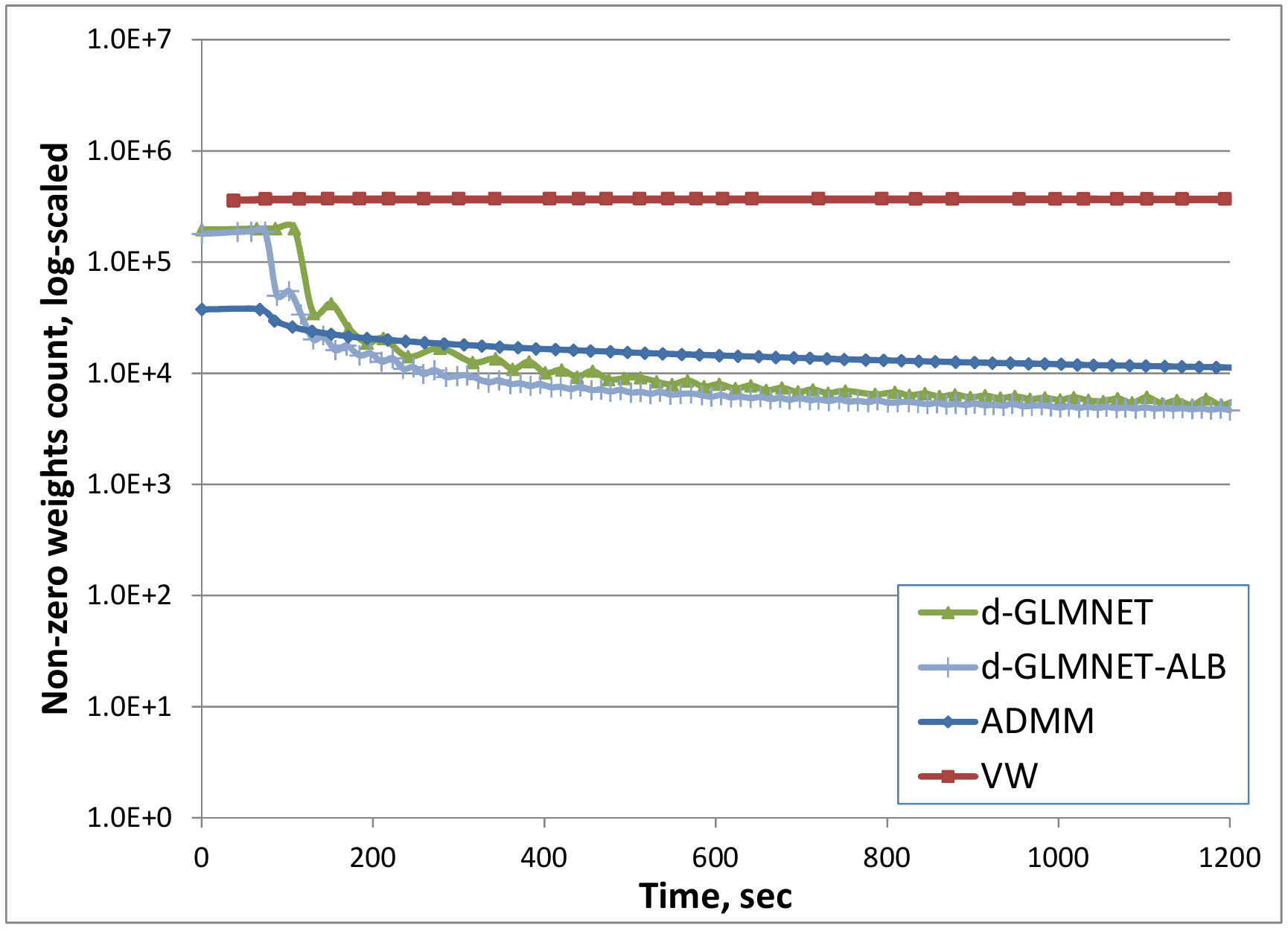}
        }
        \subfigure[yandex\_ad] {
                \includegraphics[width=0.31\textwidth]{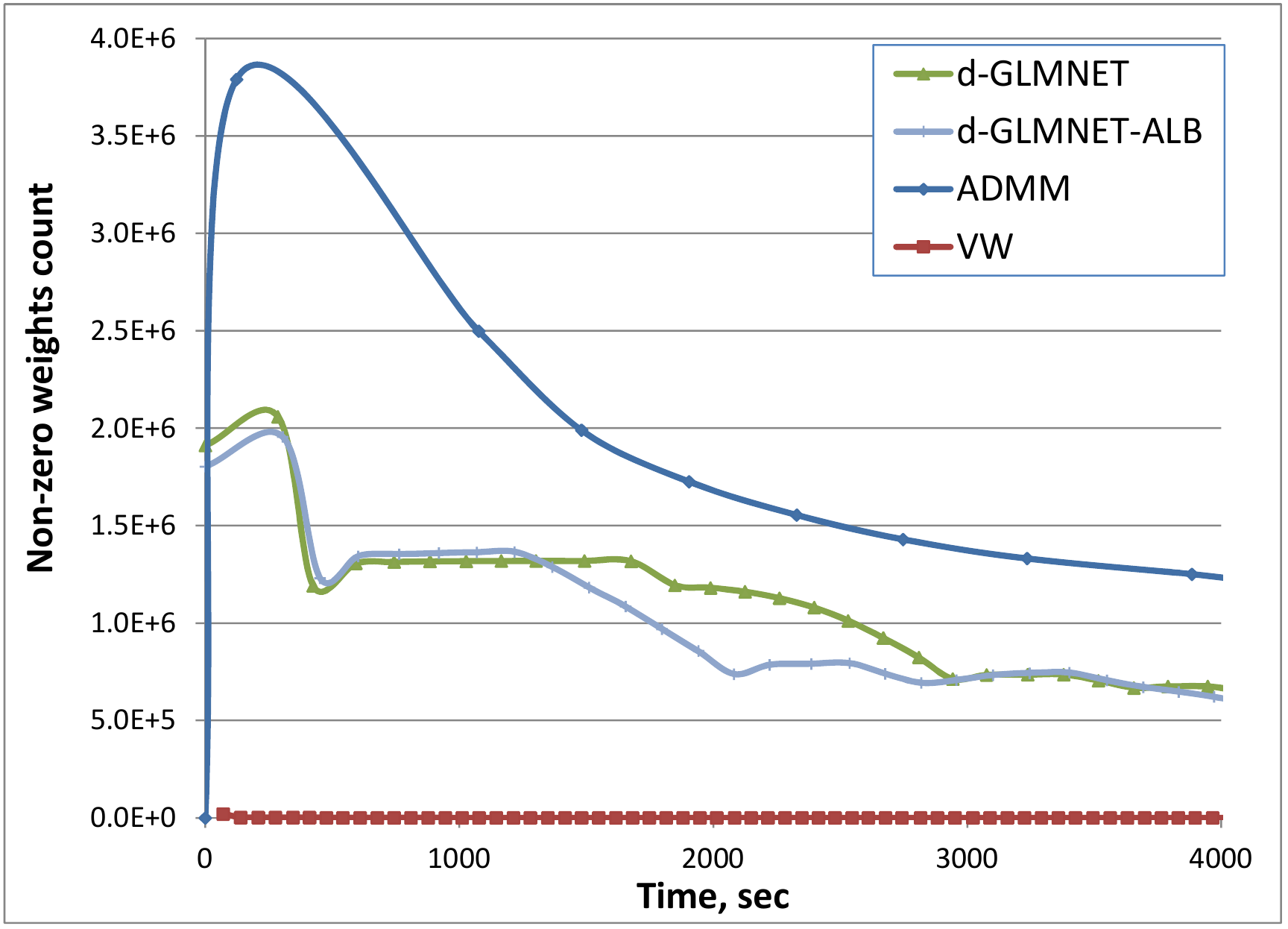}
        }
        \subfigure[epsilon] {
                \includegraphics[width=0.31\textwidth]{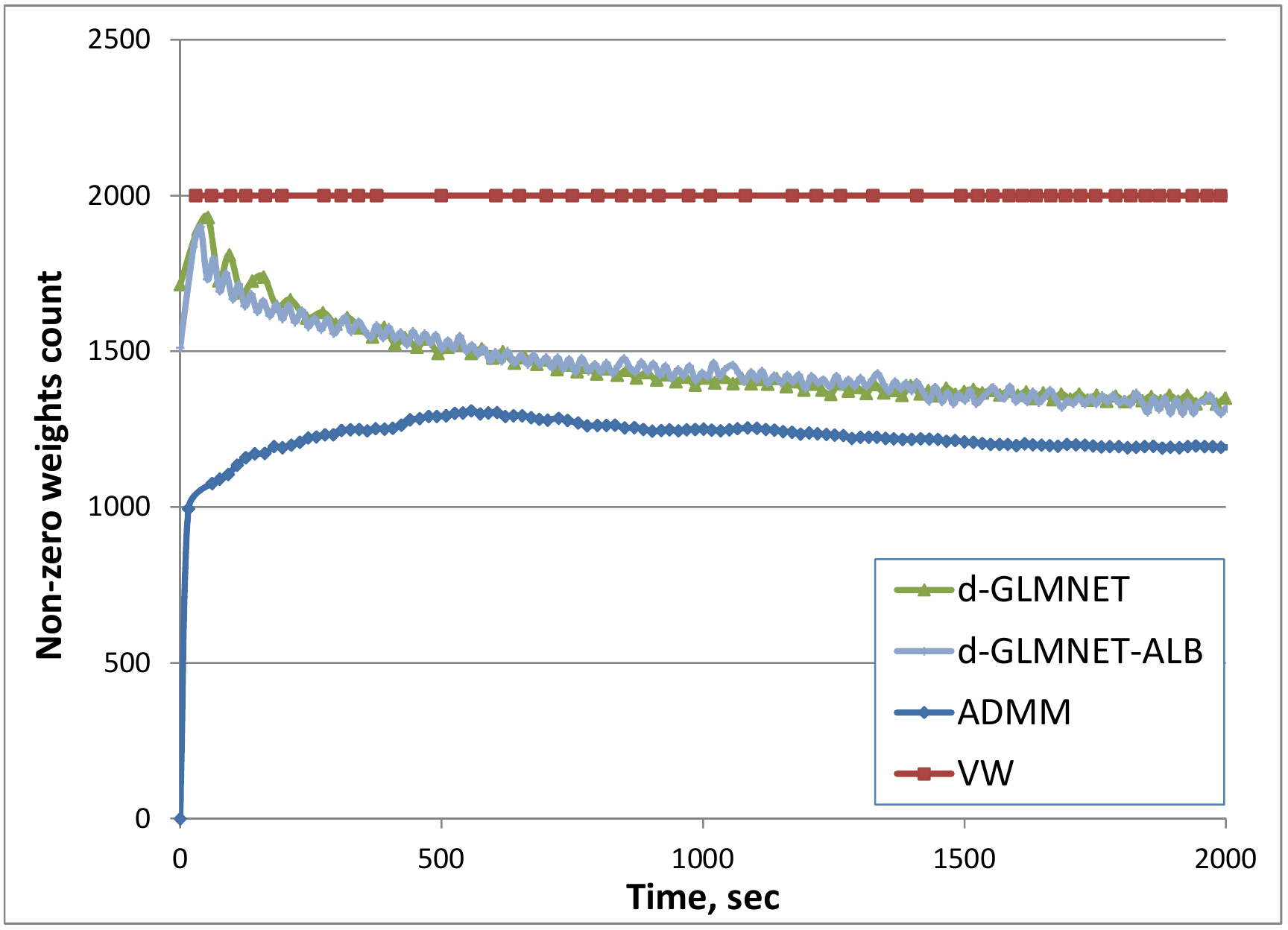}
        }
        \caption{$L_1$ regularization: number of non-zero weights vs. time.}
        \label{fig:l1-nnz}
\end{center}
\end{figure}

\begin{figure}[t]
\begin{center}
        \subfigure[webspam] {
                \includegraphics[width=0.31\textwidth]{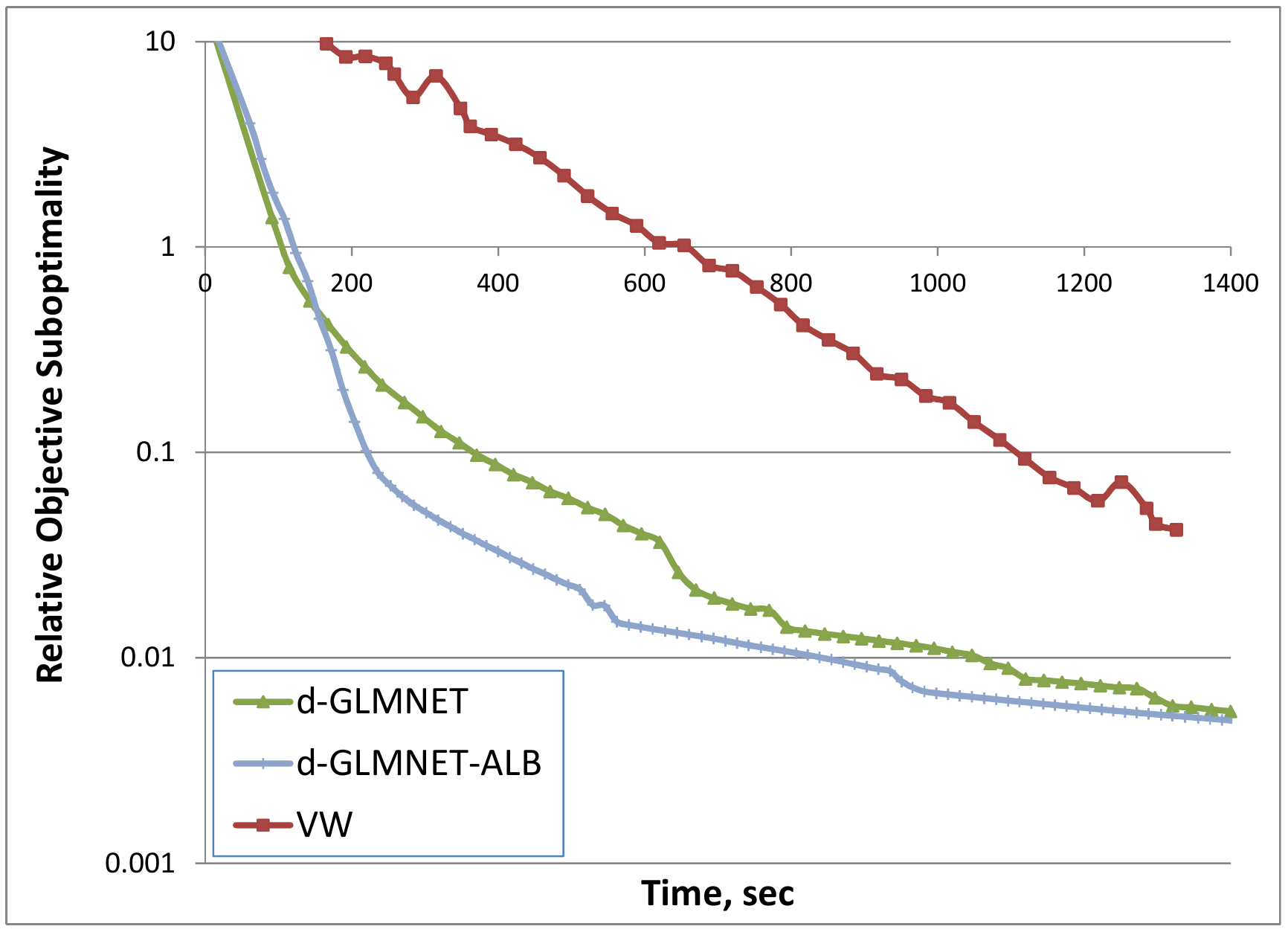}
        }
        \subfigure[yandex\_ad] {
                \includegraphics[width=0.31\textwidth]{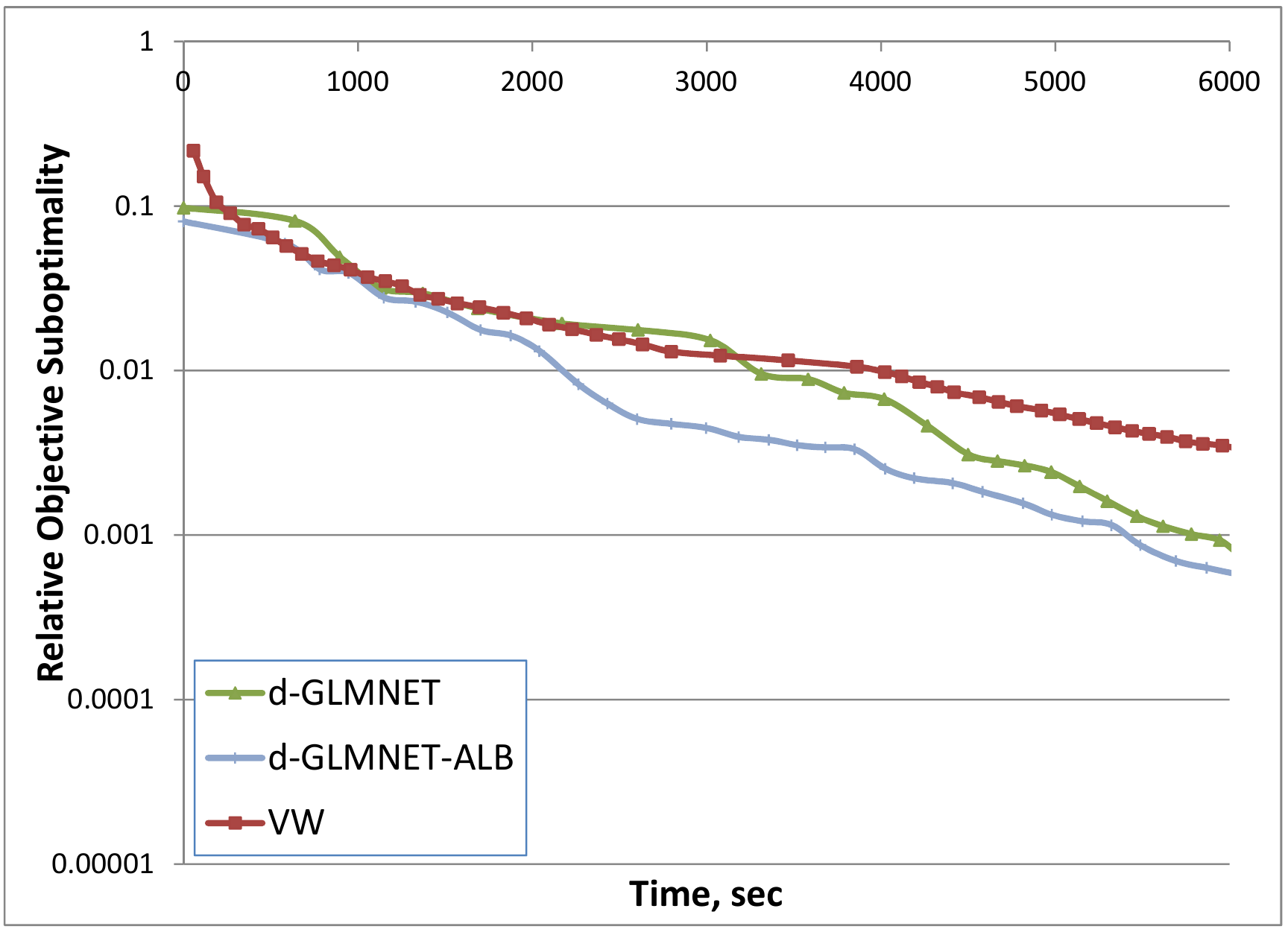}
        }
        \subfigure[epsilon] {
                \includegraphics[width=0.31\textwidth]{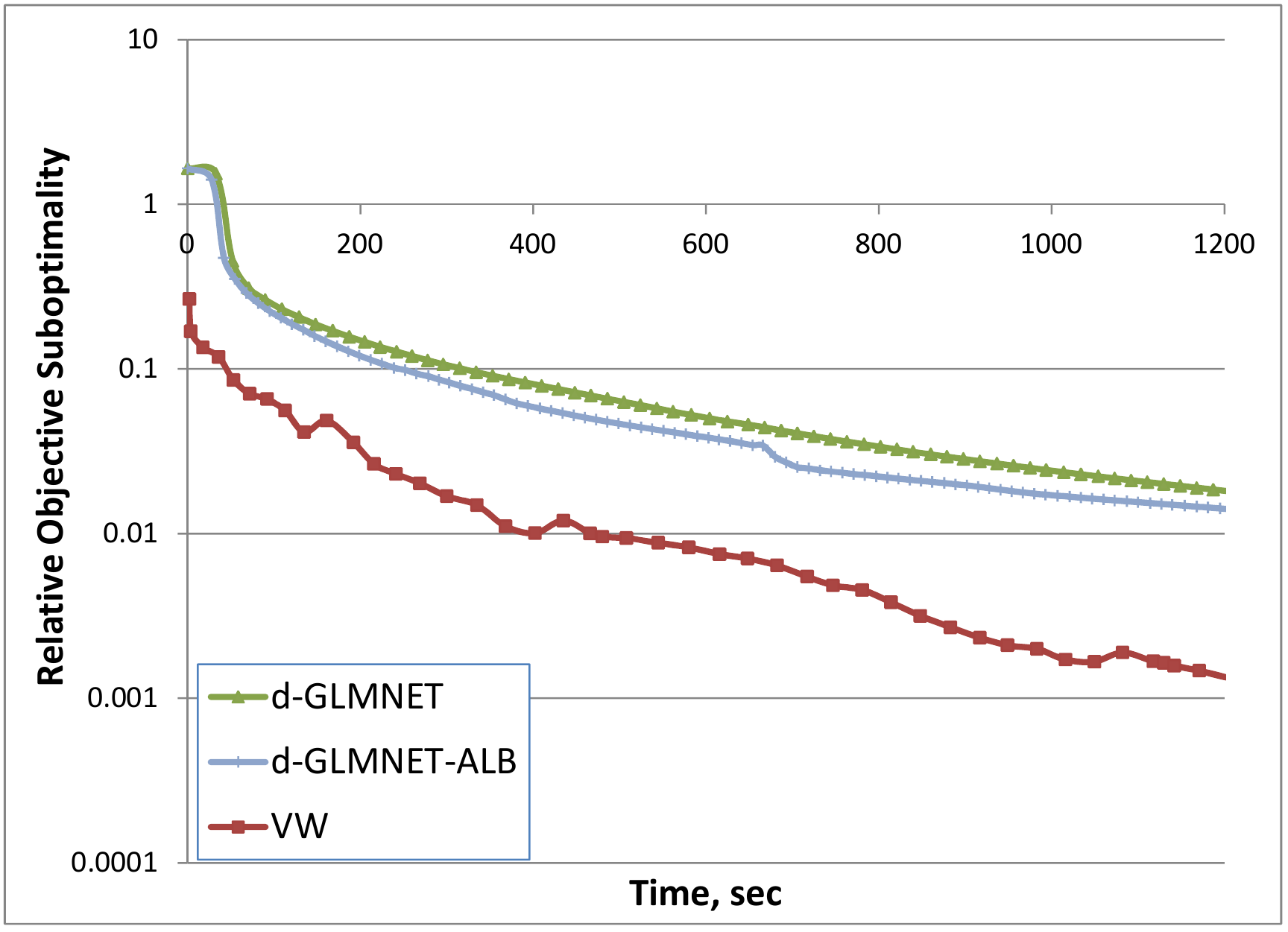}
        }
        \caption{$L_2$ regularization: relative objective suboptimality vs. time.}
        \label{fig:l2-target}
\end{center}
\begin{center}
        \subfigure[webspam] {
                \includegraphics[width=0.31\textwidth]{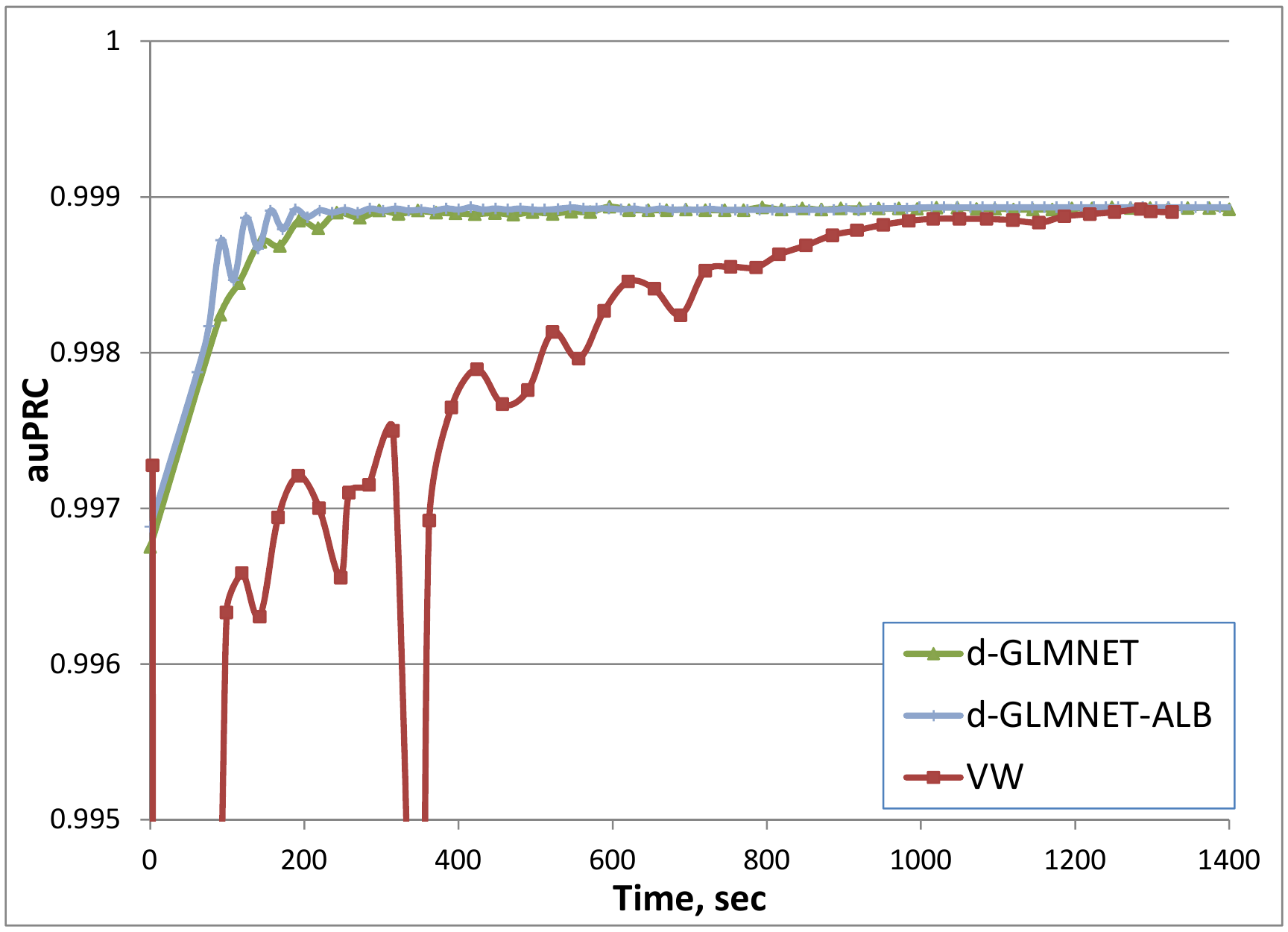}
        }
        \subfigure[yandex\_ad] {
                \includegraphics[width=0.31\textwidth]{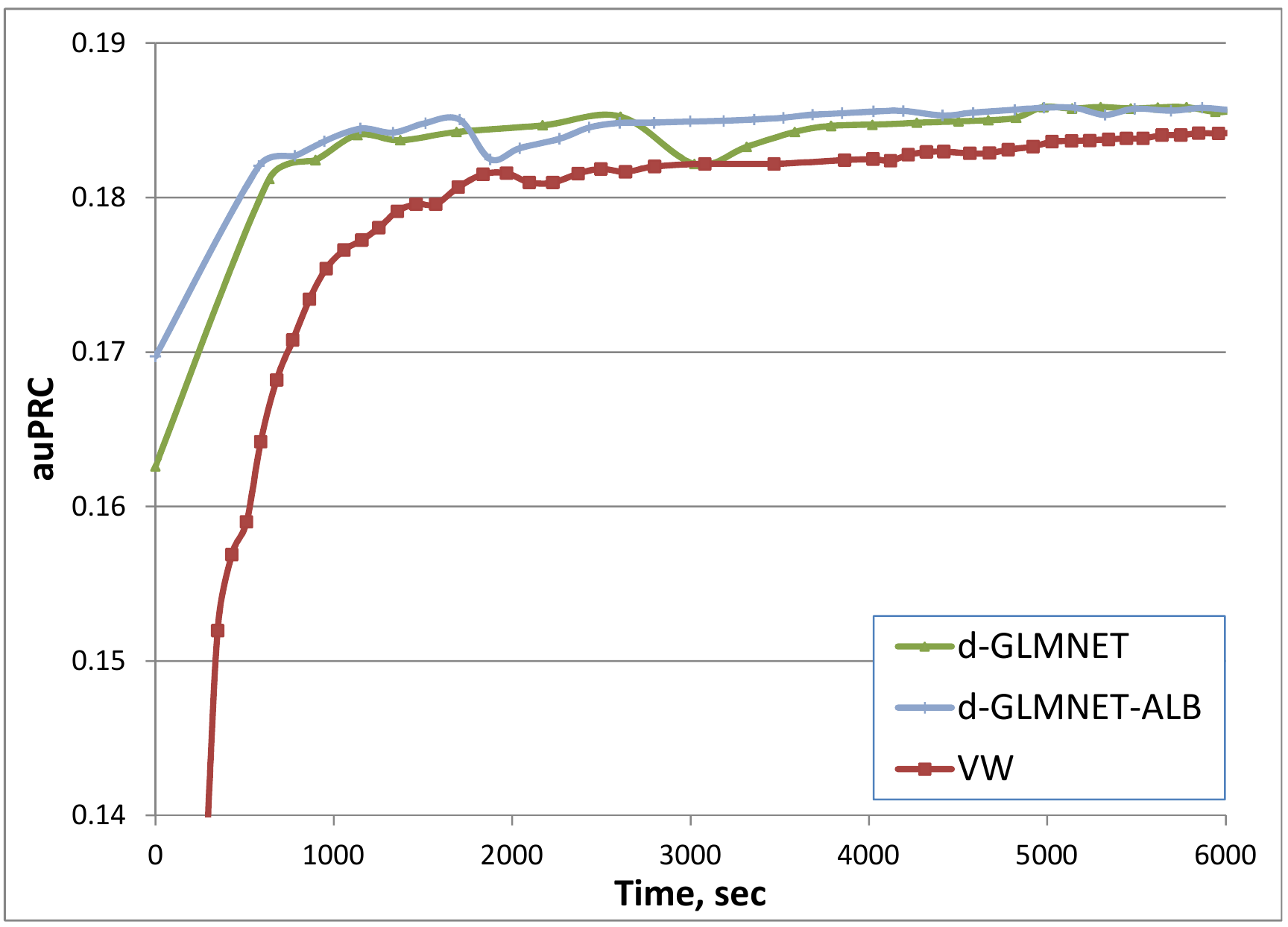}
        }
        \subfigure[epsilon] {
                \includegraphics[width=0.31\textwidth]{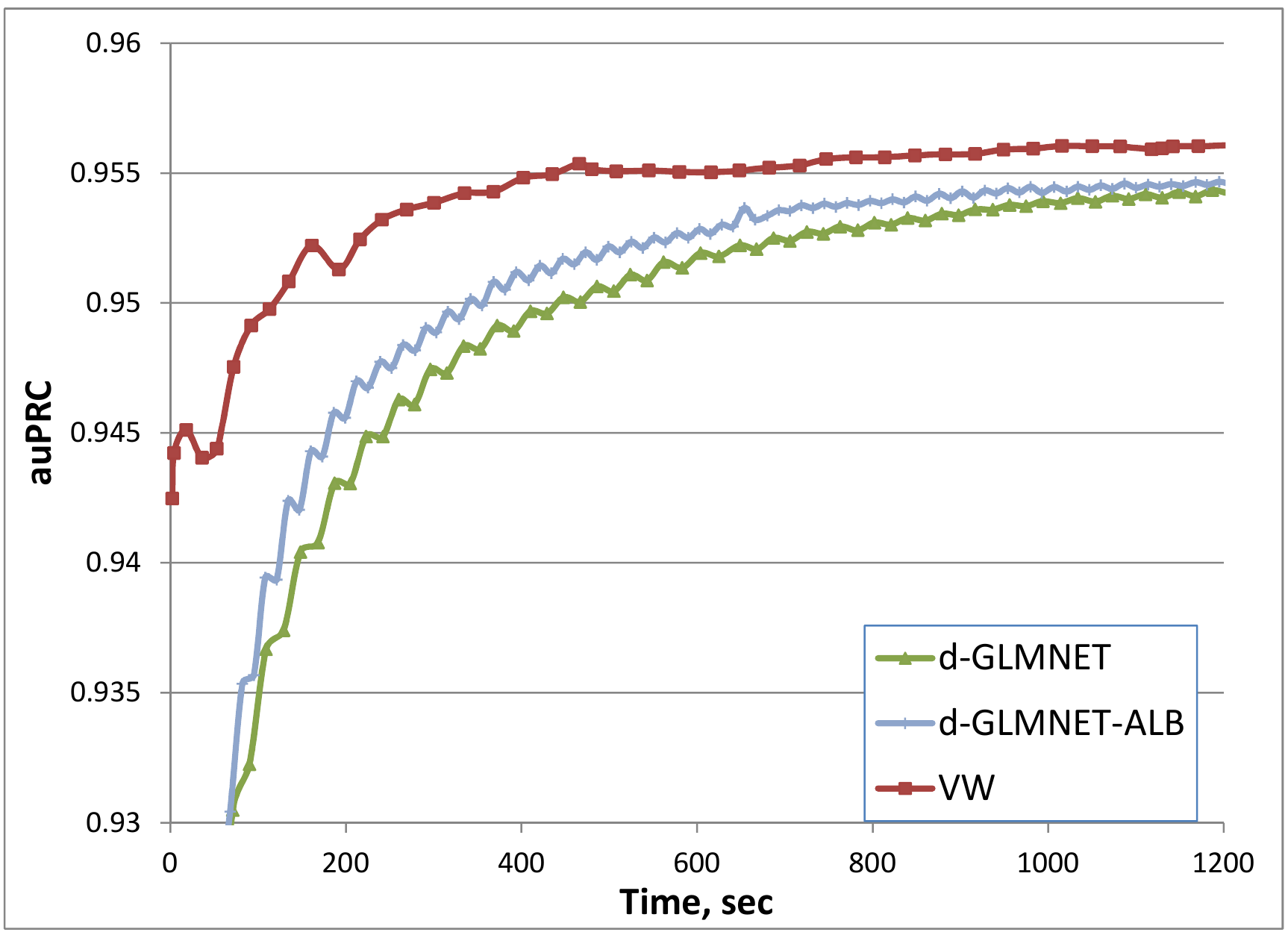}
        }
       \caption{$L_2$ regularization: testing quality (area under precision-recall curve) vs. time.}
       \label{fig:l2-testing}
\end{center}
\end{figure}

\begin{figure}[t]
\begin{center}
        \subfigure[webspam] {
                \includegraphics[width=0.31\textwidth]{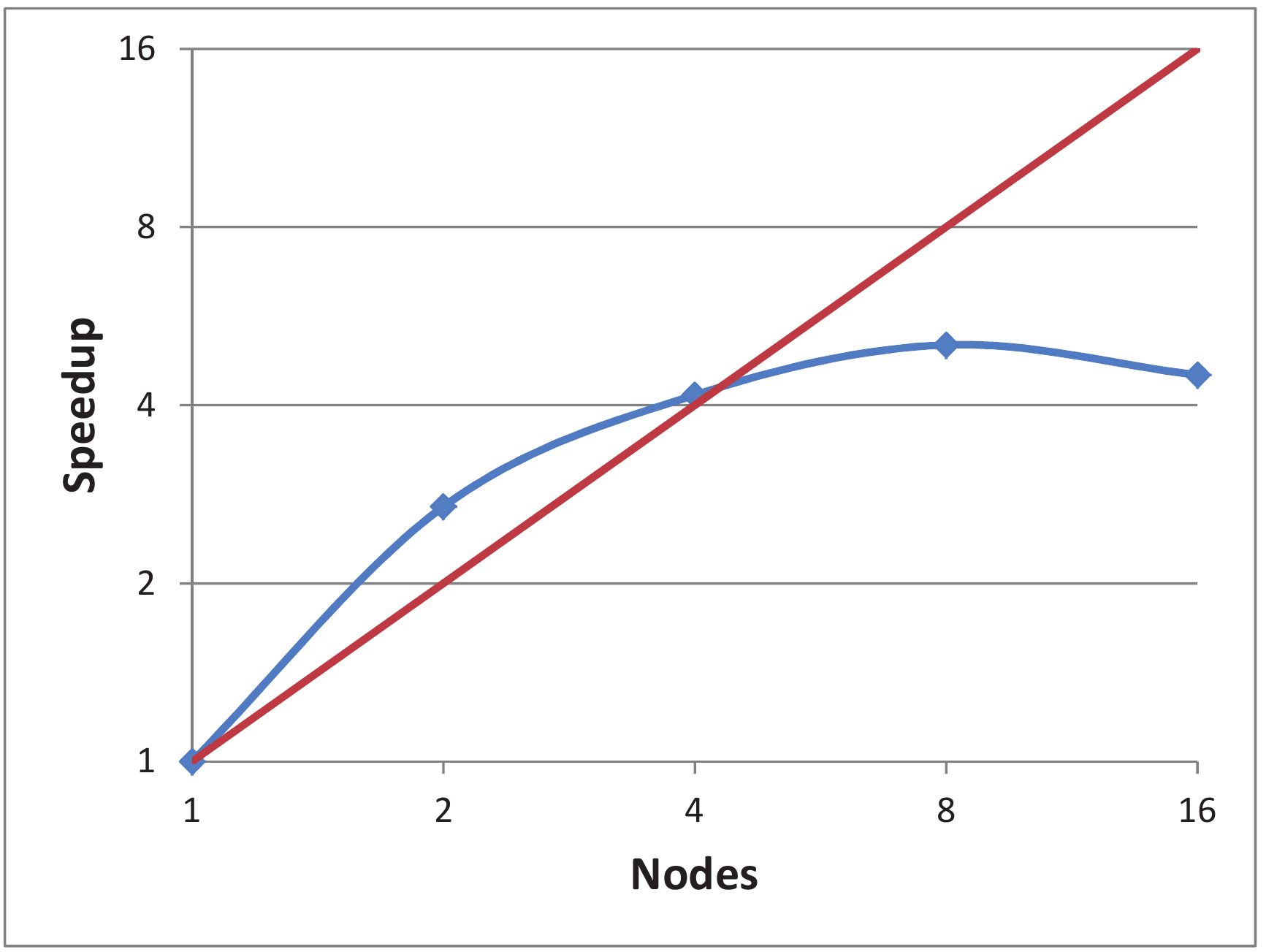}
        }
        \subfigure[yandex\_ad] {
                \includegraphics[width=0.31\textwidth]{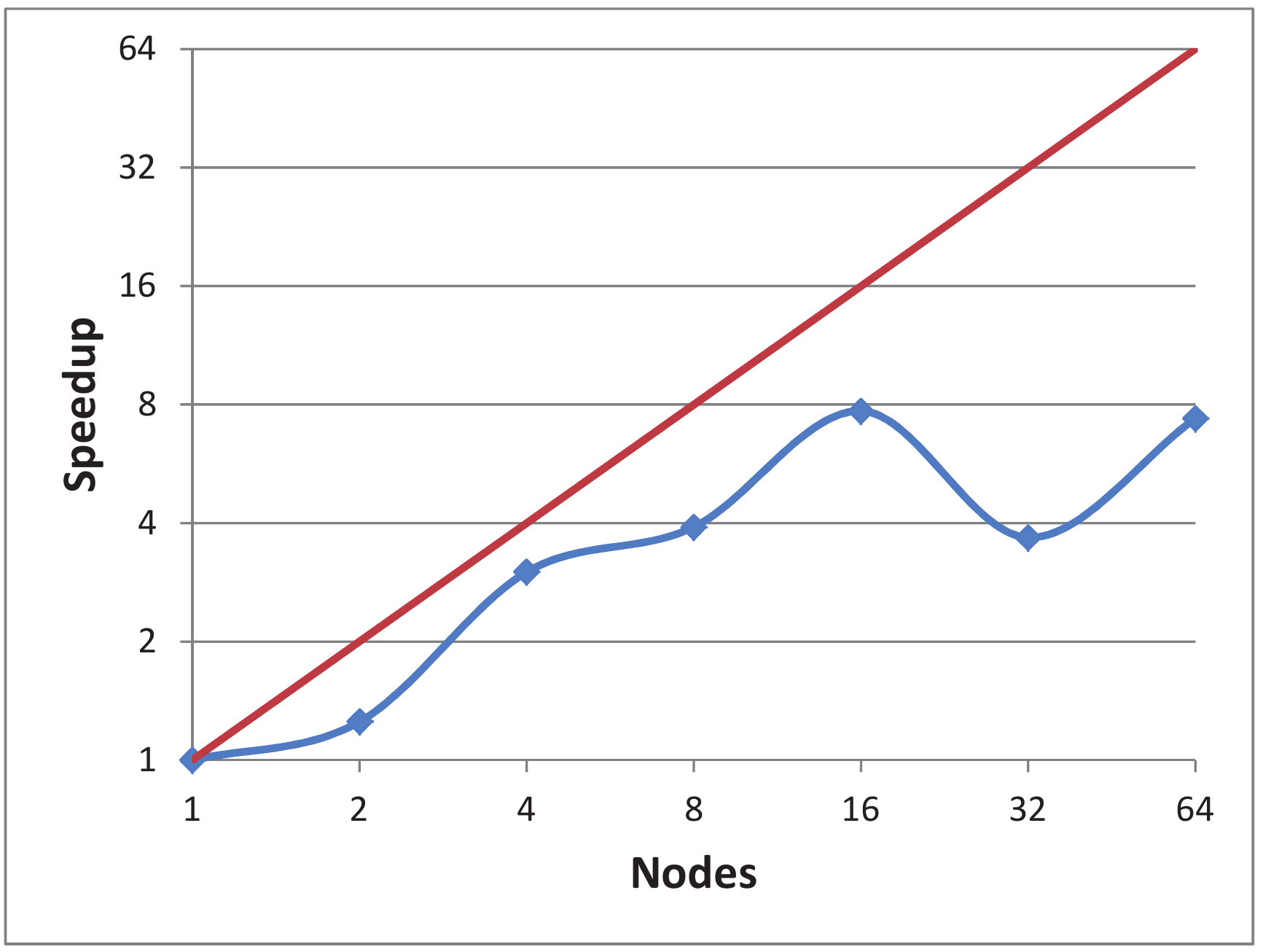}
        }
        \subfigure[epsilon] {
                \includegraphics[width=0.31\textwidth]{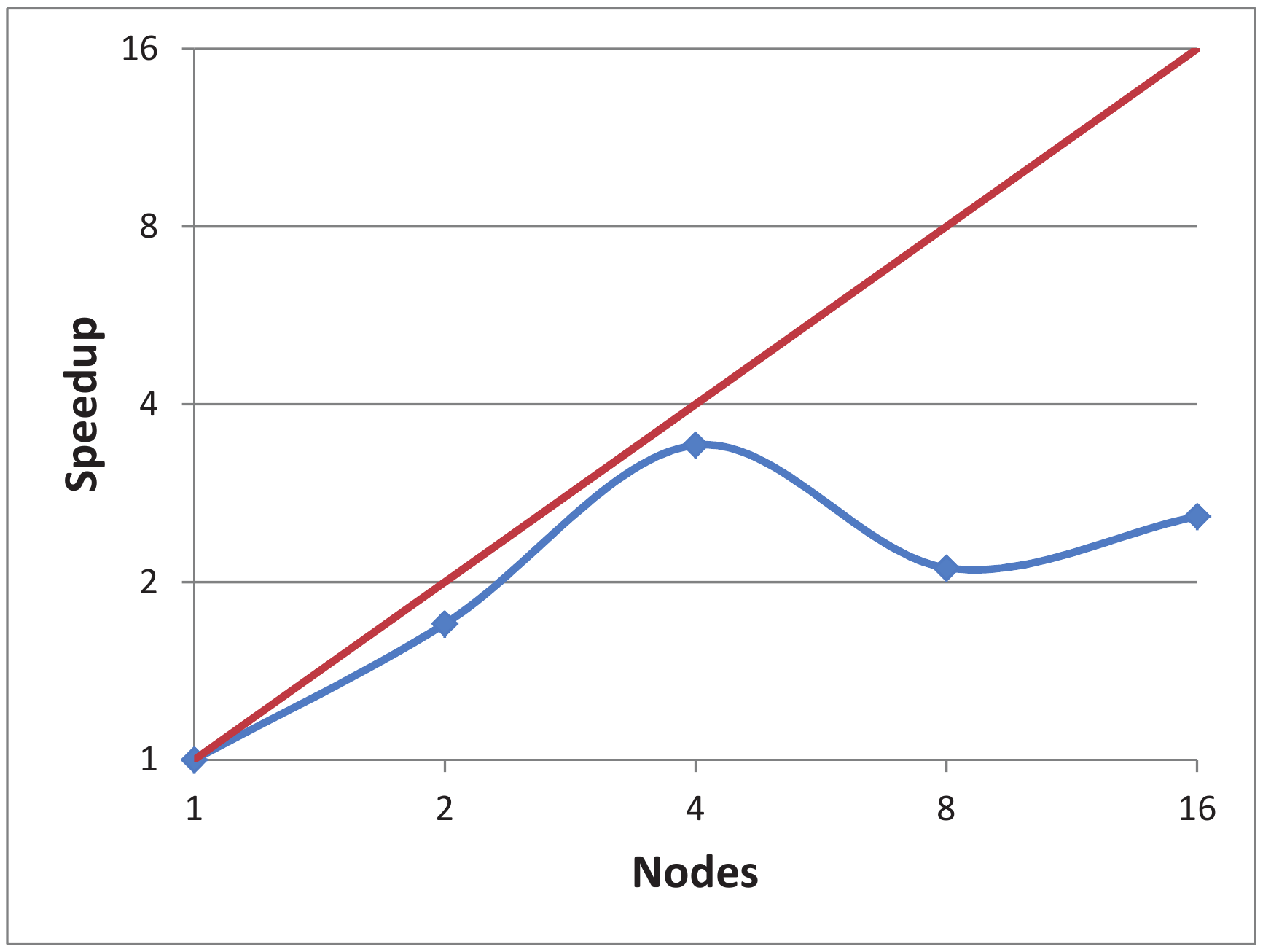}
        }
        \caption{$L_1$ regularization: relative speedup of the \texttt{d-GLMNET-ALB} algorithm for the different number of nodes: blue line. Linear speedup for reference (fictional): red line.}
        \label{fig:l1-speedup}
\end{center}
\begin{center}
        \subfigure[webspam] {
                \includegraphics[width=0.31\textwidth]{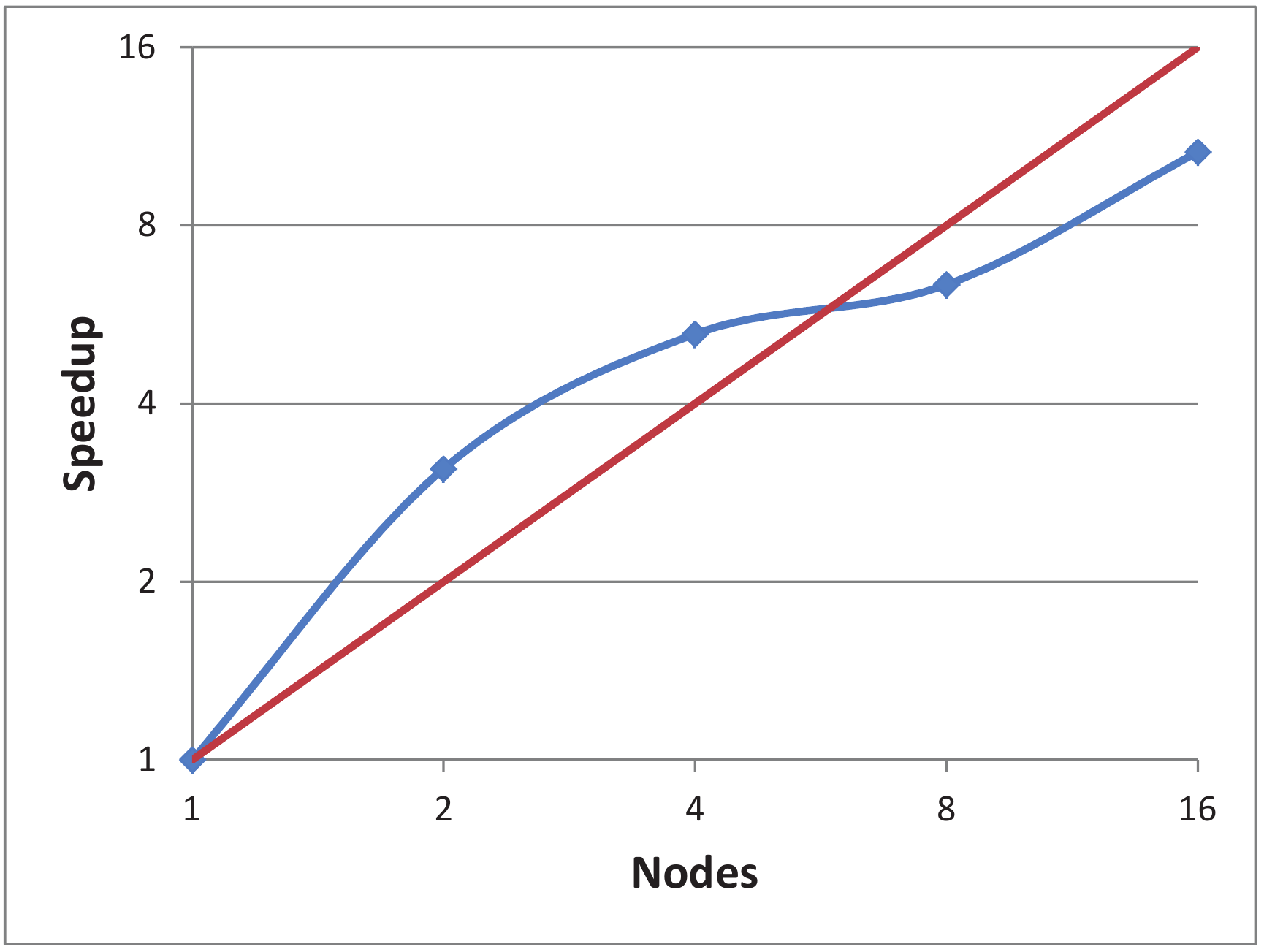}
        }
        \subfigure[yandex\_ad] {
                \includegraphics[width=0.31\textwidth]{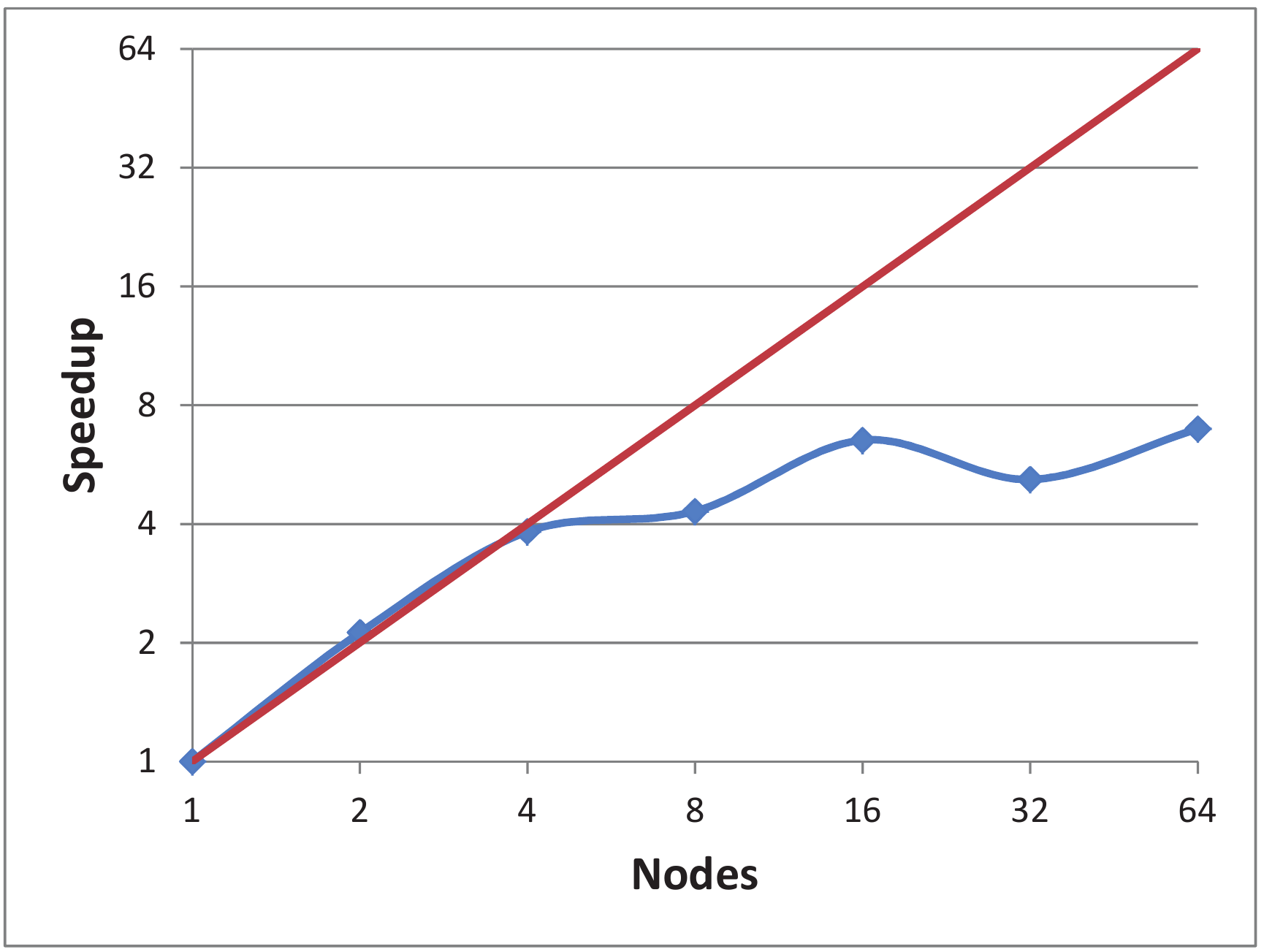}
        }
        \subfigure[epsilon] {
                \includegraphics[width=0.31\textwidth]{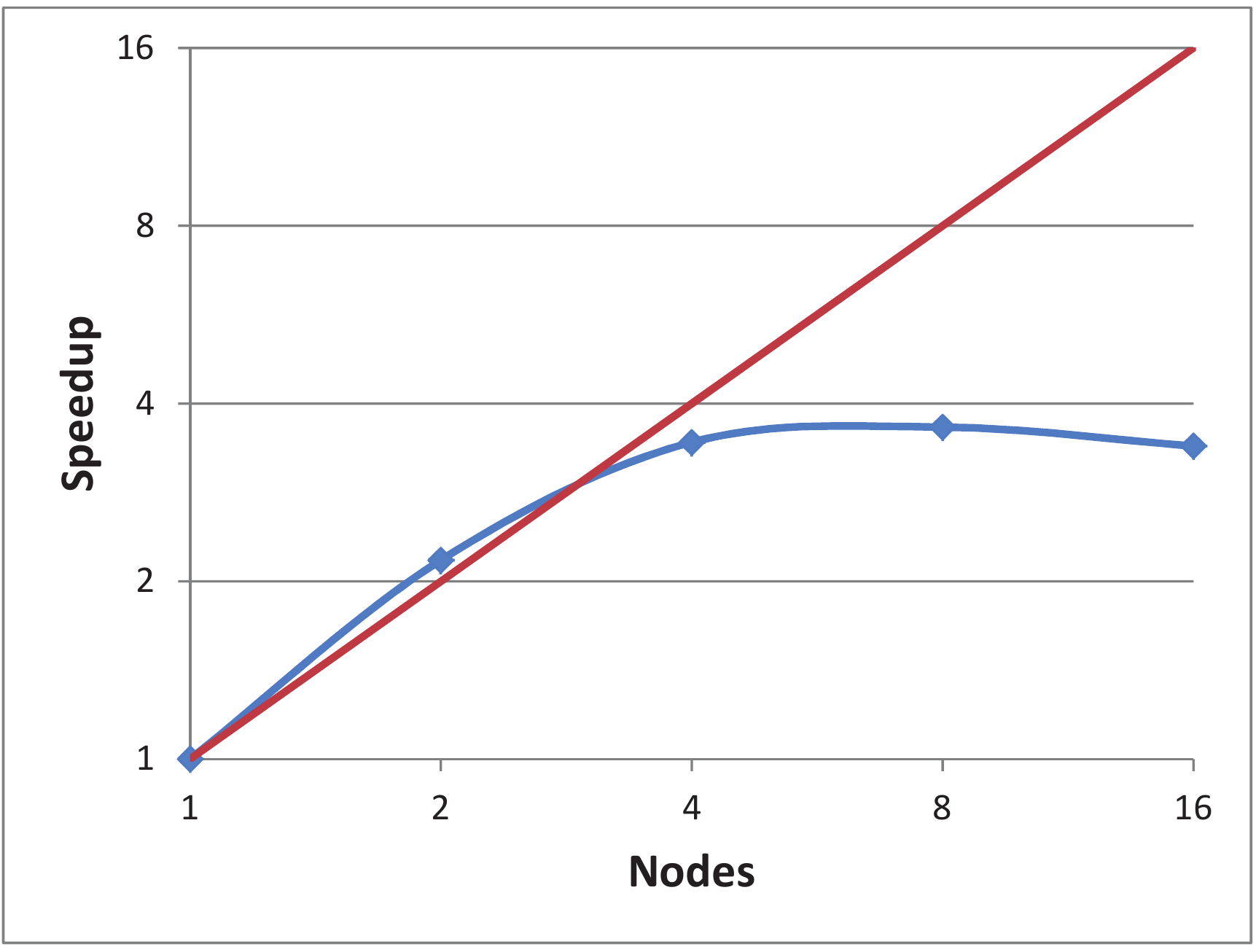}
        }
        \caption{$L_2$ regularization: relative speedup of the \texttt{d-GLMNET-ALB} algorithm for the different number of nodes: blue line. Linear speedup for reference (fictional): red line.}
        \label{fig:l2-speedup}
\end{center}
\end{figure}

\subsection{Results and discussion}

Firstly we show how adaptively changing $\mu$ parameter affects on \texttt{d-GLMNET} algorithm using ``yandex\_ad'' dataset as an example. Fig. \ref{fig:yandex_ad-mu} compares two cases :
 constant $\mu = 1$ and adaptive $\mu$. Adaptively changing $\mu$ slightly improves speed of convergence and testing accuracy but dramatically improves the sparsity.

To evaluate and compare the speed of the algorithms we created scatter plots ``Relative objective suboptimality vs. time'' (Fig. \ref{fig:l1-target} and \ref{fig:l2-target}) and ``Testing quality vs. time'' (Fig. \ref{fig:l1-testing} and \ref{fig:l2-testing}).
With $L_1$ regularization, Fig. \ref{fig:l1-target}, \ref{fig:l1-testing} shows that \texttt{d-GLMNET} algorithm has the same or faster speed of objective function optimization and improving testing accuracy on ``webspam'' and ``yandex\_ad'' datasets then competing algorithms.
The ADMM algorithm generally performs well and it is slightly better than \texttt{d-GLMNET} and \texttt{d-GLMNET-ALB} on ``epsilon'' dataset.
The \texttt{Vowpal Wabbit} program has the same or worse testing accuracy for all datasets but it poorly optimizes the objective.
Experiments with $L_1$ regularization showed that ``Asynchronous Load Balancing'' always improved or left the same the performance of the \texttt{d-GLMNET}.

Also for runs with $L_1$-regularization we created a scatter plot ``Number of non-zero weights vs. time'' (Fig. \ref{fig:l1-nnz}).
The sparsity of solutions by \texttt{d-GLMNET} is better then that of \texttt{ADMM} for ``webspam'' and ``yandex\_ad'' datasets but slightly worse for ``epsilon'' dataset.
Sparsity produced by \texttt{Vowpal Wabbit} is inconsistent: too sparse or too dense when compared to other algorithms.

With $L_2$ regularization, \texttt{d-GLMNET} optimizes the objective function faster (Fig. \ref{fig:l2-target}) and achieves better testing accuracy (Fig. \ref{fig:l2-testing}) on sparse datasets with large number of features - ``webspam'' and ``yandex\_ad''.
However on dense dataset ``epsilon'', where the number of features is relatively small, L-BFGS warmstarted by online learning is better.
Again \texttt{d-GLMNET-ALB} is faster then it's synchronous counterpart.

Evaluation of the  speed of the \texttt{d-GLMNET-ALB} algorithm with different numbers of computing nodes (Fig. \ref{fig:l1-speedup} and \ref{fig:l2-speedup}) shows that on each dataset the speedup achieved with the increased number of nodes is limited.
This happens because of two reasons. First, block-diagonal approximation of the Hessian becomes less accurate while splitting dataset over larger number of nodes, updates from nodes comes in conflict more often, so the algorithm makes smaller steps.
Second, communication cost increases.

\section{Conclusions and future work}

In this paper we presented a novel architecture for training generalized linear models with regularization in the distributed setting
based on parallel coordinate descent.
We implemented a novel parallel coordinate descent algorithm  \texttt{d-GLMNET} and its modification \texttt{d-GLMNET-ALB}, which is immune to the ``slow node problem''.
We proposed a trust-region update which yields a sparse solution in case of $L_1$ regularization.
In a series of numerical experiments we demonstrated that our algorithms and software implementation are well suited for training logistic regression with $L_1$ and $L_2$ regularization on the large scale.
Experiments show that \texttt{d-GLMNET} is superior over several state-of-the-art algorithms when training on sparse high-dimensional datasets.
It possesses a faster convergence speed and enjoys speedup when using multiple computing nodes.
This is essential for large-scale machine learning problems where long training time is often an issue.

\texttt{d-GLMNET} can also be extended to regularizers other than $L_1$ and $L_2$. Optimizing quadratic approximation (\ref{quad}) over one weight $\Delta \beta_j$ via any one-dimensional optimization algorithm is simple enough; it can be done either exactly of approximately for any separable regularizer: bridge, SCAD, e.t.c.

Suchard et al. \cite{Suchard2013} showed that training GLMs on multicore gives significant speedup.
Combining computations on multicore on each node with the distributed architecture is a promising direction for further development.


\begin{acknowledgments}
We would like to thank John Langford for the advices on Vowpal Wabbit and Ilya Muchnik for his continuous support.

\end{acknowledgments}

\appendix


\section{Avoiding line search}
\label{app:large_mu}

\begin{proposition}
When $\mu \ge \frac{\Lambda_{max}}{(1 - \sigma)\lambda_{min}}$ the Armijo rule (\ref{armijo}) with $\gamma = 0$ will be satisfied for $\alpha = 1$.
\end{proposition}
\begin{proof}
Let $g(t) = L(\vbeta + t \Delta \vbeta)$. Then
\begin{align*}
g'(t) = \nabla L(\vbeta + t \Delta \vbeta)^T \Delta \vbeta, \\
g''(t) = \Delta \vbeta^T \nabla^2 L(\vbeta + t \Delta \vbeta) \Delta \vbeta.
\end{align*}
We obtain the upper bound for $L(\vbeta + \Delta \vbeta)$
\begin{align*}
L(\vbeta + \Delta \vbeta) & = g(1) = g(0) + \int_{0}^{1} g'(t) dt \le g(0) + \int_{0}^{1} \left( g'(0) + t \max_{z \in [0, 1]} | g''(z) | \right) dt \\
& = L(\vbeta) + \nabla L(\vbeta)^T \Delta \vbeta + \frac{1}{2} \max_{z \in [0, 1]} | \Delta \vbeta^T \nabla^2 L(\vbeta + z \Delta \vbeta) \Delta \vbeta | \\
& \le L(\vbeta) + \nabla L(\vbeta)^T \Delta \vbeta +  \frac{1}{2} \Lambda_{max} \| \Delta \vbeta \|^2.
\end{align*}
In the last inequality we used $\nabla^2 L(\vbeta) \preceq \Lambda_{max} I$ which follows from (\ref{full-hessian-limit}).
Then
\begin{align}
\label{f_delta_beta_limit}
f(\vbeta + \Delta \vbeta^*) - f(\vbeta) = L(\vbeta + \Delta \vbeta^*) - L(\vbeta) + R(\Delta \vbeta  + \vbeta^*) - R(\vbeta) \notag \\
\le \nabla L(\vbeta)^T \Delta \vbeta +  \frac{1}{2} \Lambda_{max} \| \Delta \vbeta \|^2 + R(\vbeta + \Delta \vbeta^*) - R(\vbeta)
= D + \frac{1}{2} \Lambda_{max} \| \Delta \vbeta \|^2.
\end{align}

Where we used $D$ from Armijo rule (\ref{armijo}) for a particular case $\gamma = 0$
$$
D = \nabla L(\vbeta)^T \Delta \vbeta + R(\vbeta + \Delta \vbeta^*) - R(\vbeta).
$$
Since $\Delta \vbeta^*$ minimizes (\ref{quad-appr-mod})
\begin{align*}
\nabla L(\vbeta)^T \Delta \vbeta + R(\vbeta + \Delta \vbeta^*) + \frac{1}{2} (\Delta \vbeta^*) (\mu (\tilde{H}(\vbeta) + \nu I)) \Delta \vbeta^* \le R(\vbeta),
\end{align*}
then $D$ has the upper bound
\begin{align*}
\nabla L(\vbeta)^T \Delta \vbeta + R(\vbeta + \Delta \vbeta^*)- R(\vbeta) & \le - \frac{1}{2} (\Delta \vbeta^*) (\mu (\tilde{H}(\vbeta)) + \nu I) \Delta \vbeta^* \\
D & \le  - \frac{1}{2} (\Delta \vbeta^*) \mu (\tilde{H}(\vbeta) + \nu I) \Delta \vbeta^*.
\end{align*}
By noticing that $\lambda_{min} I \preceq \tilde{H}(\vbeta) + \nu I$ we obtain for $\mu \ge \frac{\Lambda_{max}}{(1 - \sigma)\lambda_{min}}$:
\begin{align}
\label{lambda_max_limit}
\frac{1}{2} \Lambda_{max} \| \Delta \vbeta^* \|^2 \le \frac{1}{2} (1 - \sigma) \mu \lambda_{min} \| \Delta \vbeta^* \|^2 & \le \frac{1}{2} (1 - \sigma) (\Delta \vbeta^*)^T
(\mu (\tilde{H} + \nu I)) \Delta \vbeta^* \notag \\
& \le - (1 - \sigma) D.
\end{align}

Substituting (\ref{lambda_max_limit}) into (\ref{f_delta_beta_limit}) yields
$$
f(\vbeta + \Delta \vbeta^*) - f(\vbeta) \le D - (1 - \sigma) D = \sigma D,
$$
which proves that Armijo rule is satisfied for $\alpha = 1$.
\end{proof}

\section{Loss functions second derivative upper bounds}
\label{app:second_der}

\begin{itemize}
\item \textbf{Squared loss}: $\ell(y, \hat{y}) = \frac{1}{2} (y - \hy)^2$, $\frac{\partial^2 \ell(y, \hy)}{\partial \hy^2} = 1$.
\item \textbf{Logistic loss}: $\ell(y, \hat{y}) = \log(1 + \exp(-y \hat{y}))$.
For logistic loss $\frac{\partial^2 \ell(y, \hy)}{\partial \hy^2} = p(\hy) (1 - p(\hy))$, where $p(\hy) = 1 / (1 + e^{-\hy})$ and consequently
$\frac{\partial^2 \ell(y, \hy)}{\partial \hy^2} \le \frac{1}{4}$.

\item \textbf{Probit loss}: $\ell(y, \hat{y}) = - \log(\Phi(y\hat{y}))$, where $\Phi(\cdot)$ is a CDF of a normal distribution.
Denote $p(\hat{y}) = \frac{1}{\sqrt{2 \pi}} \exp\left(-\hat{y}^2 / 2 \right)$.
It is sufficient to give the proof only for $y = 1$ because $\frac{\partial^2 \ell(-1, \hy)}{\partial \hy^2} = \frac{\partial^2 \ell(1, -\hy)}{\partial \hy^2}$.
We have
\begin{align*}
\frac{\partial^2 \ell(y, \hat{y})}{\partial \hat{y}^2} & = \frac{\hy p(\hat{y})}{\Phi(\hat{y})} + \frac{p^2(\hat{y})}{\Phi^2(\hat{y})}.
\end{align*}
When $\hy \ge 0$ second derivative has upper bound
$$
\frac{\hy p(\hat{y})}{\Phi(\hat{y})} + \frac{p^2(\hat{y})}{\Phi^2(\hat{y})} \le 2 \hy p(\hat{y}) + 4 p^2(\hat{y}) \le 2 p(1) + 4 p(0),
$$
because $\Phi(\hat{y}) \ge \Phi(0) = 1/2$ and $\hy p(\hat{y})$ reaches maximum in $\hy = 1$. When $\hy \in (-1, 0)$ second derivative is bounded.
The case $\hy \le -1$ is a bit more complex. From \cite{normal_cdf_bounds} we have
$$
\frac{|\hy|p(\hy)}{1 + \hy^2} < \Phi(\hy) < \frac{p(\hy)}{|\hy|},
$$
then
\begin{align*}
\frac{1}{\Phi(\hy)} < \frac{1 + \hy^2}{|\hy|p(\hy)}, \\
\frac{\hy}{\Phi(\hy)} <  \hy\frac{|\hy|}{p(\hy)},
\end{align*}
and finally
\begin{align*}
\frac{\partial^2 \ell(y, \hat{y})}{\partial \hat{y}^2} = \frac{\hy p(\hat{y})}{\Phi(\hat{y})} + \frac{p^2(\hat{y})}{\Phi^2(\hat{y})} < \hy|\hy| + \left(\frac{1 + \hy^2}{|\hy|} \right)^2 = -\hy^2 + \frac{1 + 2 \hy^2 + \hy^4}{\hy^2} = 2 + \frac{1}{\hy^2} \le 3.
\end{align*}
Thus for all cases $\hy \ge 0$, $\hy \in (-1, 0)$, $\hy \le -1$ second derivative has the upper bound.
\end{itemize}

\section{Area under Precision-Recall curve (auPRC)}
\label{sec:auprc}
Area under Precision-Recall curve is a classification quality measure.
Consider $n$ examples with binary class labels $y_i \in \{-1, +1\}$ and a classifier predictions with a real-valued outcomes $p_i \in [0, 1]$.
Given a threshold $a$ precision (Pr) and recall (Rc) are defined as follows
\begin{align*}
Pr(a) = \frac{|\{i \suchthat p_i \ge a \;\&\; y_i = +1\}|}{|\{i \suchthat p_i \ge a \}|}, \\
Rc(a) = \frac{|\{i \suchthat p_i \ge a \;\&\; y_i = +1\}|}{|\{i \suchthat y_i = +1\}|}.
\end{align*}
Precision-Recall curve is obtained by varying $a \in [0, 1]$. The area under this curve is considered a classification quality measure.
It is more sensitive than a commonly used ROC AUC in case of highly imbalanced classes \cite{davis2006relationship}.

\begin{thebibliography}{0}
\expandafter\ifx\csname natexlab\endcsname\relax\def\natexlab#1{#1}\fi
\expandafter\ifx\csname bibnamefont\endcsname\relax
  \def\bibnamefont#1{#1}\fi
\expandafter\ifx\csname bibfnamefont\endcsname\relax
  \def\bibfnamefont#1{#1}\fi
\expandafter\ifx\csname citenamefont\endcsname\relax
  \def\citenamefont#1{#1}\fi
\expandafter\ifx\csname url\endcsname\relax
  \def\url#1{\texttt{#1}}\fi
\expandafter\ifx\csname urlprefix\endcsname\relax\def\urlprefix{URL }\fi
\providecommand{\bibinfo}[2]{#2}
\providecommand{\eprint}[2][]{\url{#2}}

\end{thebibliography}


\begin{thebibliography}{99}

\providecommand{\url}[1]{\normalfont{#1}}
\providecommand{\urlprefix}{Available from: }

\bibitem{Genkin2007}
Genkin, A., Lewis, D. D., Madigan, D. (2007). Large-scale Bayesian logistic regression for text categorization. Technometrics, 49(3), 291-304.

\bibitem{Mcmahan2013}
McMahan, H. B., Holt, G., Sculley, D., Young, M., Ebner, D., Grady, J., Nie, Davydov, E., Golovin, D., Chikkerur, S., Liu, D., Wattenberg, M., Hrafnkelsson, A.M.,
Boulos, T., Kubica, J.  (2013). Ad click prediction: a view from the trenches. In Proceedings of the 19th ACM SIGKDD international conference on Knowledge discovery and data mining (pp. 1222-1230). ACM.

\bibitem{Yuan2010}
Yuan, G. X., Chang, K. W., Hsieh, C. J., Lin, C. J. (2010). A comparison of optimization methods and software for large-scale l1-regularized linear classification. The Journal of Machine Learning Research, 11, 3183-3234.

\bibitem{Bradley2011}
Bradley, J. K., Kyrola, A., Bickson, D., \& Guestrin, C. (2011). Parallel coordinate descent for l1-regularized loss minimization. arXiv preprint arXiv:1105.5379.

\bibitem{Yuan2012Recent}
Yuan, G. X., Ho, C. H., \& Lin, C. J. (2012). Recent advances of large-scale linear classification. Proceedings of the IEEE, 100(9), 2584-2603.

\bibitem{yuan2006model}
Yuan, M., \& Lin, Y. (2006). Model selection and estimation in regression with grouped variables. Journal of the Royal Statistical Society: Series B (Statistical Methodology), 68(1), 49-67.

\bibitem{Meier2008}
Meier, L., Van De Geer, S., \& Buhlmann, P. (2008). The group lasso for logistic regression. Journal of the Royal Statistical Society: Series B (Statistical Methodology), 70(1), 53-71.

\bibitem{Fu1998}
Fu, W. J. (1998). Penalized regressions: the bridge versus the lasso. Journal of computational and graphical statistics, 7(3), 397-416.

\bibitem{Fan2001}
Fan, J., \& Li, R. (2001). Variable selection via nonconcave penalized likelihood and its oracle properties. Journal of the American statistical Association, 96(456), 1348-1360.

\bibitem{Sutton2002}
Sutton, C., \& McCallum, A. (2006). An introduction to conditional random fields for relational learning. Introduction to statistical relational learning, 93-128.

\bibitem{Yu2011}
Yu, H. F., Huang, F. L., \& Lin, C. J. (2011). Dual coordinate descent methods for logistic regression and maximum entropy models. Machine Learning, 85(1-2), 41-75.

\bibitem{Karampatziakis2010}
Karampatziakis, N., \& Langford, J. (2010). Online importance weight aware updates. arXiv preprint arXiv:1011.1576.

\bibitem{Friedman2010}
Friedman, J., Hastie, T., \& Tibshirani, R. (2010). Regularization paths for generalized linear models via coordinate descent. Journal of statistical software, 33(1), 1.

\bibitem{Yuan2012a}
Yuan, G. X., Ho, C. H., \& Lin, C. J. (2012). An improved glmnet for l1-regularized logistic regression. The Journal of Machine Learning Research, 13(1), 1999-2030.

\bibitem{Balakrishnan2007}
Balakrishnan, S., \& Madigan, D. (2008). Algorithms for sparse linear classifiers in the massive data setting. The Journal of Machine Learning Research, 9, 313-337.

\bibitem{Langford2009}
Langford, J., Li, L., \& Zhang, T. (2009). Sparse online learning via truncated gradient. In Advances in neural information processing systems (pp. 905-912).

\bibitem{Mcmahan2011}
McMahan, H. B. (2011). Follow-the-Regularized-Leader and Mirror Descent: Equivalence Theorems and L1 Regularization. In AISTATS (pp. 525-533).

\bibitem{Agarwal2011}
Agarwal, A., Chapelle, O., Dudik, M., \& Langford, J. (2014). A reliable effective terascale linear learning system. The Journal of Machine Learning Research, 15(1), 1111-1133.

\bibitem{Peng2013}
Peng, Z., Yan, M., \& Yin, W. (2013). Parallel and distributed sparse optimization. In Signals, Systems and Computers, 2013 Asilomar Conference on (pp. 659-646). IEEE.

\bibitem{Zinkevich2010}
Zinkevich, M., Weimer, M., Li, L., \& Smola, A. J. (2010). Parallelized stochastic gradient descent. In Advances in neural information processing systems (pp. 2595-2603).

\bibitem{Ho2013}
Ho, Q., Cipar, J., Cui, H., Lee, S., Kim, J. K., Gibbons, P. B., \& Xing, E. P. (2013). More effective distributed ml via a stale synchronous parallel parameter server. In Advances in neural information processing systems (pp. 1223-1231).

\bibitem{Richtarik2012}
Richt\'{a}rik, P., \& Tak\'a\v{c}, M. (2015). Parallel coordinate descent methods for big data optimization. Mathematical Programming, 1-52.

\bibitem{Tseng2007}
Tseng, P., \& Yun, S. (2009). A coordinate gradient descent method for nonsmooth separable minimization. Mathematical Programming, 117(1-2), 387-423.

\bibitem{Dean2004}
Dean~J, Ghemawat~S. {MapReduce : Simplified Data Processing on Large Clusters}.
  In: OSDI' 04. San Francisco; 2004.

\bibitem{Low2010}
Low, Y., Gonzalez, J. E., Kyrola, A., Bickson, D., Guestrin, C. E., \& Hellerstein, J. (2014). Graphlab: A new framework for parallel machine learning. arXiv preprint arXiv:1408.2041.

\bibitem{Zaharia2010}
Zaharia, M., Chowdhury, M., Franklin, M. J., Shenker, S., \& Stoica, I. (2010). Spark: Cluster Computing with Working Sets. HotCloud, 10, 10-10.

\bibitem{Kumar2014}
Kumar, A., Beutel, A., Ho, Q., \& Xing, E. P. (2014). Fugue: Slow-worker-agnostic distributed learning for big models on big data.

\bibitem{smola2010architecture}
Smola, A., \& Narayanamurthy, S. (2010). An architecture for parallel topic models. Proceedings of the VLDB Endowment, 3(1-2), 703-710.

\bibitem{Boyd2010}
Boyd, S., Parikh, N., Chu, E., Peleato, B., \& Eckstein, J. (2011). Distributed optimization and statistical learning via the alternating direction method of multipliers. Foundations and Trends in Machine Learning, 3(1), 1-122.

\bibitem{Suchard2013}
Suchard, M. A., Simpson, S. E., Zorych, I., Ryan, P., \& Madigan, D. (2013). Massive parallelization of serial inference algorithms for a complex generalized linear model. ACM Transactions on Modeling and Computer Simulation (TOMACS), 23(1), 10.

\bibitem{nocedal1980updating}
Nocedal, J. (1980). Updating quasi-Newton matrices with limited storage. Mathematics of computation, 35(151), 773-782.

\bibitem{davis2006relationship}
Jesse, D., \& Mark, G. (2006). The relationship between precision-recall and ROC curves, 233-240. In Proceedings of the 23rd International Conference on Machine Learning. Association for Computing Machinery, New York, NY.

\bibitem{Beck2009}
Beck, A., \& Teboulle, M. (2009). A fast iterative shrinkage-thresholding algorithm for linear inverse problems. SIAM journal on imaging sciences, 2(1), 183-202.

\bibitem{Smith2015}
Smith, V., Forte, S., Jordan, M. I., \& Jaggi, M. (2015). L1-Regularized Distributed Optimization: A Communication-Efficient Primal-Dual Framework. arXiv preprint arXiv:1512.04011.

\bibitem{normal_cdf_bounds}
\url{http://schools-wikipedia.org/wp/n/Normal_distribution.htm}

\bibitem{Zhuang2015}
Zhuang, Y., Chin, W. S., Juan, Y. C., \& Lin, C. J. (2015). Distributed Newton Methods for Regularized Logistic Regression. In Advances in Knowledge Discovery and Data Mining (pp. 690-703). Springer International Publishing.

\end{thebibliography}
\end{document}